\documentclass[accepted]{uai2022} 

\usepackage[american]{babel}

\usepackage[round]{natbib} 
\usepackage{mathtools} 
\usepackage{booktabs} 
\usepackage{tikz} 
\usepackage{amsthm,amsmath,amsfonts}
\usepackage{color}
\usepackage{algorithm}
\usepackage{algorithmic}
\usepackage{subcaption}
\usepackage{hyperref}
\usepackage{comment}

\newtheorem{theorem}{Theorem}
\newtheorem{definition}[theorem]{Definition}
\newtheorem{lemma}[theorem]{Lemma}

\newtheorem{problem}[theorem]{Problem}
\DeclareMathOperator*{\dom}{dom}
\DeclareMathOperator*{\range}{range}
\DeclareMathOperator*{\Dither}{Dither}
\DeclareMathOperator*{\argmax}{\arg\,\max}

\newcommand{\bbE}{\mathbb{E}}
\newcommand{\bi}{\mathbf{i}}
\newcommand{\bp}{\mathbf{p}}
\newcommand{\bx}{\mathbf{x}}
\newcommand{\bz}{\mathbf{z}}
\newcommand{\calA}{\mathcal{A}}
\newcommand{\calM}{\mathcal{M}}
\newcommand{\calT}{\mathcal{T}}

\newcommand{\bin}{b_{\text{in}}}
\newcommand{\bout}{b_{\text{out}}}
\newcommand{\Bin}{B_{\text{in}}}
\newcommand{\Bout}{B_{\text{out}}}

\def\calX{\mathcal{X}}

\title{Privacy-Aware Compression for Federated Data Analysis}

\author[1]{Kamalika Chaudhuri*}
\author[1]{Chuan Guo*}
\author[1]{Mike Rabbat}
\affil[1]{%
    Meta AI, USA. *Equal contribution.
}

\begin{document}
\maketitle



\begin{abstract}
Federated data analytics is a framework for distributed data analysis where a server compiles noisy responses from a group of distributed low-bandwidth user devices to estimate aggregate statistics. Two major challenges in this framework are privacy, since user data is often sensitive, and compression, since the user devices have low network bandwidth. Prior work has addressed these challenges separately by combining standard compression algorithms with known privacy mechanisms.
In this work, we take a holistic look at the problem and design a family of privacy-aware compression mechanisms that work for any given communication budget. We first propose a mechanism for transmitting a single real number that has optimal variance under certain conditions. We then show how to extend it to metric differential privacy for location privacy use-cases, as well as vectors, for application to federated learning. Our experiments illustrate that our mechanism can lead to better utility vs. compression trade-offs for the same privacy loss in a number of settings. \end{abstract}

\section{Introduction}





Federated data analytics is a framework for distributed data analysis and machine learning that is widely applicable to use-cases involving continuous data collection from a large number of devices. Here, a central server receives responses from a large number of distributed clients, and aggregates them to compute a global statistic or a machine learning model. An example is training and fine-tuning a speech-to-text model for a digital assistant; here a central server has a speech-to-text model, which is continuously updated based on feedback from client devices about the quality of predictions on their local data. Another example is maintaining real-time traffic statistics in a city for ride-share demand prediction; here, a central server located at the ride-share company collects and aggregates location data from a large number of user devices. 



Most applications of federated data analysis involve two major challenges --  privacy and compression. Since typical use-cases involve personal data from users, it is important to maintain their privacy. This is usually achieved by applying a local differentially private (LDP) algorithm~\citep{duchi2013local, kasiviswanathan2011can} on the raw inputs at the client device so that only sanitized data is transmitted to the server. Additionally, since the clients frequently have low-bandwidth high-latency uplinks, it is also important to ensure that they communicate as few bits to the server as possible. Most prior work in this area~\citep{girgis2021shuffled, kairouz2021distributed, agarwal2021skellam} addressed these two challenges separately -- first, a standard LDP algorithm is used to sanitize the client responses, and then standard compression procedures are used to compress them before transmission. However, this leads to a loss in accuracy of the client responses, ultimately leading to a loss in estimation or learning accuracy at the server. Moreover, each of these methods requires a very specific communication budget and is not readily adapted to other budgets.

In this work, we take a closer look at the problem and propose designing the privacy mechanism in conjunction with the compression procedure. To this end, we propose a formal property called {\em{asymptotic consistency}} that any private federated data analysis mechanism should possess. Asymptotic consistency requires that the aggregate statistics computed by the server converge to the non-private aggregate statistics as the number of clients grows. If the server averages the client responses, then a sufficient condition for asymptotic consistency is that the clients send an unbiased estimate of their input. Perhaps surprisingly, many existing mechanisms are not unbiased, and thus not asymptotically consistent.

We first consider designing such unbiased mechanisms that, given any communication budget $b$, transmit a continuous scalar value that lies in the interval $[0, 1]$ with local differential privacy and no public randomness. We observe that many existing methods, such as truncated Gaussian, lead to biased solutions and asymptotically inconsistent outcomes if the inputs lie close to an end-point of the truncation interval. Motivated by this, we show how to convert two existing local differentially private mechanisms for transmitting categorical values -- bit-wise randomized response~\citep{warner1965randomized} and generalized randomized response -- to unbiased solutions.

We then propose a novel mechanism, the Minimum Variance Unbiased (MVU) mechanism, that given $b$ bits of communication, exploits the ordinal nature of the inputs to provide a better privacy-accuracy trade-off. We show that if the input is drawn uniformly from the set $\{0, 1/(2^b - 1), \ldots, 1 \}$, then the MVU mechanism has minimum variance among all mechanisms that satisfy the local differential privacy constraints.  We show how to adapt this mechanism to metric differential privacy~\citep{andres2013geo} for location privacy applications. To adapt it to differentially private SGD (DP-SGD; ~\citet{abadi2016deep}), we then show how to extend it to vectors within an $L_p$-ball, and establish tight privacy composition guarantees. 

Finally, we investigate the empirical performance of the MVU mechanism in two concrete use-cases: distributed mean estimation and private federated learning. In each case, we compare our method with several existing baselines, and show that our mechanism can achieve better utility for the same privacy guarantees. In particular, we show that the MVU mechanism can match the performance of specially-designed gradient compression schemes such as stochastic signSGD~\citep{jin2020stochastic} for DP-SGD training of neural networks at the same communication budget.



\section{Preliminaries}

In private federated data analysis, a central server calculates aggregate statistics based on sensitive inputs from $n$ clients. The statistics might be as simple as the prevalence of some event, or as complicated as a gradient to a large neural network. To preserve privacy, the clients transmit a sanitized version of their input to the server. Two popular privacy notions used for sanitization are local differential privacy~\citep{duchi2013local, kasiviswanathan2011can} and metric differential privacy~\citep{andres2013geo}.  

\subsection{Privacy Definitions}


\begin{definition}
A randomized mechanism $\calM$ with domain $\dom(\calM)$ and range $\range(\calM)$ is said to be $\epsilon$-local differentially private (LDP) if for all pairs $x$ and $x'$ in the domain of $\calM$ and any $S \subseteq \range(\calM)$, we have that:
\[ \Pr(\calM(x) \in S) \leq e^{\epsilon} \Pr(\calM(x') \in S). \]
\end{definition}
Here $\epsilon$ is a privacy parameter where lower $\epsilon$ implies better privacy. The LDP mechanism $\calM$ is run on the client side, and the result is transmitted to the server. We assume that the clients and the server do not share any randomness. It might appear that a local DP requirement implies that a client's response contains very little useful information. While each individual response may be highly noisy, the server is still able to obtain a fairly accurate estimate of an \emph{aggregate property} if there are enough clients. Thus, the challenge in private federated data analysis is to design protocols --- privacy mechanisms for clients and aggregation algorithms for servers --- so that client privacy is preserved, and the server can obtain an accurate estimate of the desired statistic. 

A related definition is metric differential privacy (metric-DP)~\citep{chatzikokolakis2013broadening}, which is also known as geo-indistinguishability~\citep{andres2013geo} and is commonly used to quantify location privacy.

\begin{definition}\label{def:metricdp}
A randomized mechanism $\calM$ with domain $\dom(\calM)$ and range $\range(\calM)$ is said to be $\epsilon$-metric DP with respect to a metric $d$ if for all pairs $x$ and $x'$ in the domain of $\calM$ and any $S \subseteq \range(\calM)$, we have that:
\[ \Pr(\calM(x) \in S) \leq e^{\epsilon d(x, x')} \Pr(\calM(x') \in S). \]
\end{definition}

Metric DP offers granular privacy that is quantified by the metric $d$ -- inputs $x$ and $x'$ that are close in $d$ are indistinguishable, while those that are far apart in $d$ are less so. 

\subsection{Problem Statement}

In addition to balancing privacy and accuracy, a bottleneck of federated analytics is communication since client devices typically have limited network bandwidth. Thus, the goal is to achieve privacy and accuracy along with a limited amount of communication between clients and servers. We formalize this problem as follows.  

\begin{problem}\label{prob:fl}
Suppose we have $n$ clients with sensitive data $x_1, \ldots, x_n$ where each $x_i$ lies in a domain $\calX$, and a central server $S$ seeks to approximate an aggregate statistic $\calT_n$. Our goal is to design two algorithms, a client-side mechanism $\calM$ and a server-side aggregation procedure $\calA_n$, such that the following conditions hold:
\begin{enumerate}
\item $\calM$ is $\epsilon$-local DP (or $\epsilon$-metric DP). 
\item The output of $\calM$ can be encoded in $b$ bits.
\item $\calA_n(\calM(x_1), \ldots, \calM(x_n))$ is a good approximation to $\calT_n(x_1, \ldots, x_n)$. 
\end{enumerate}
\end{problem}

Prior works addressed the communication challenge by making the clients use a standard local DP mechanism followed by a standard quantization process. We develop methods where both mechanisms are designed together so as to obtain high accuracy at the server end.


\subsection{Asymptotic Consistency}

We posit that any good federated analytics solution $(\calM, \calA_n)$ where $\calM$ is a client mechanism and $\calA_n$ is the server-side aggregation procedure should have an {\em{asymptotic consistency}} property. Loosely speaking, this property ensures that the server can approximate the target statistic $\calT_n$ arbitrarily well with clients. Formally,

\begin{definition}
We say that a private federated analytics protocol is {\em{asymptotically consistent}} if the output of the server's aggregation algorithm $\calA_n( \calM(x_1), \ldots, \calM(x_n))$ approaches the target statistic $\calT_n(x_1, \ldots, x_n)$ as $n \rightarrow \infty$. In other words, for any $\alpha, \delta > 0$, there exists an $n_0$ such that for all $n \geq n_0$, we have:
\[ \Pr(| \calA_n(\calM(x_1), \ldots, \calM(x_n)) - \calT_n(x_1, \ldots, x_n)| \geq \alpha) \leq \delta \]
\end{definition}

While the server can use any aggregation protocol $\calA_n$, the most common is a simple averaging of the client responses --  $\calA_n(\calM(x_1), \ldots, \calM(x_n)) = \frac{1}{n} \sum_i \calM(x_i)$. It is easy to show the following lemma.


\begin{lemma} \label{lem:unbiased}
If $\calM(x)$ is unbiased for all $x$ and has bounded variance, and if $\calA_n$ computes the average of the client responses, then the federated analytics solution is asymptotically consistent.
\end{lemma}

While asymptotic consistency may seem basic, it is surprisingly not satisfied by a number of simple solutions. An example is when $\calM(x)$ is a Gaussian mechanism whose output is truncated to an interval $[a, b]$. In this case, if $x_i = a$ for all $i$, the truncated Gaussian mechanism will be biased with $\bbE[\calM(x_i)] > x_i$, and consequently the server's aggregate will not approach $a$ for any number of clients.

Some of the recently proposed solutions for federated learning are also not guaranteed to be asymptotically consistent. Examples include the truncated Discrete Gaussian mechanism~\citep{canonne2020discrete, kairouz2021distributed} as well as the Skellam mechanism~\citep{agarwal2021skellam}. While these mechanisms are unbiased if the range is unbounded and there are no communication constraints, their results do become biased after truncation. 


\subsection{Compression Tool: Dithering}

A core component of our proposed mechanisms is dithering -- a popular approach to quantization with a long history of use in communications~\citep{Schuchman1964dither,Gray1993dithered}, signal processing~\citep{Lipshitz1992quantization}, and more recently for communication-efficient distributed learning~\citep{Alistarh2017qsgd,Shlezinger2020uveqfed}. Suppose our goal is to quantize a scalar value $x \in [0,1]$ with a communication budget of $b$ bits. We consider the $B=2^b$ points $G=\{0, \frac{1}{B-1}, \frac{2}{B-1}, \dots, 1\}$ as the quantization lattice; \emph{i.e.}, the $B$ points uniformly spaced by $\Delta = 1/(B-1)$. Dithering can be seen as a random quantization function $\Dither : [0,1] \rightarrow G$ that is unbiased, \emph{i.e.}, $\bbE[\Dither(x)] = x$.\footnote{When the number of grid points $B$ is clear from the context, we simply write $\Dither(x)$ to simplify notation; otherwise we write $\Dither_B(x)$ to indicate the value of $B$.}  Moreover, the distribution of the quantization errors $\Dither(x) - x$ can be made independent of the distribution of $x$.

While there are many forms of dithered quantization~\citep{Gray1993dithered}, we focus on the following. If $x \in [\frac{i}{B-1}, \frac{i+1}{B-1})$ where $0 \leq i \leq B-1$, then $\Dither(x) = \frac{i}{B-1}$ with probability $(B-1) (\frac{i+1}{B-1} - x)$, and $\Dither(x) = \frac{i+1}{B-1}$ with probability $(B-1)(x - \frac{i}{B-1})$. A simple calculation shows that $\bbE[\Dither(x)] = x$ and moreover that the variance is bounded above by $\bbE[(\Dither(x) - x)^2] \le \Delta^2 / 4$. This procedure is equivalent to the \emph{non-subtractive} dithering scheme $\Dither(x) = \min_{q \in G} |q - (x - U)|$, where $U$ is uniformly distributed over the interval $[-\Delta/2, \Delta/2]$; see, e.g.,~\cite[Lemma~2]{Aysal2008distributed}. 





\section{Scalar Mechanisms}
\label{sec:scalar}

We consider Problem~\ref{prob:fl} when the input $x_i$ is a scalar in the interval $[0, 1]$, and the statistic\footnote{To simplify notation, we drop the subscript $n$ from statistics $\calT_n$ and aggregation functions $\calA_n$, when the number of clients $n$ is clear from the context.} $\calT$ is the average $\frac{1}{n}\sum_{i=1}^{n} x_i$. Our server side aggregation protocol will also output an average of the client responses. Our goal now is to design a client-side mechanism $\calM$ that is $\epsilon$-local DP, unbiased, and can be encoded in $b$ bits.

\paragraph{Notation.} The inputs to our client-side mechanism $\calM$ are: a continuous value $x \in [0, 1]$, a privacy parameter $\epsilon$ and a communication budget $b$. The output is a number $i \in \{0, \ldots, B - 1 \}$ where $B = 2^b$, represented as a sequence of $b$ bits. Additionally, we have an alphabet $A = \{ a_0, \ldots, a_{B - 1}\}$ shared between the clients and server; a number $i$ transmitted by a client is decoded as the letter $a_i$ in $A$. The purpose of $A$ is to ensure unbiasedness.

\subsection{Strategy Overview}

\begin{algorithm}[t]
\caption{Strategy for privacy-aware compression}
\label{alg:overview}
\begin{algorithmic}[1]
\STATE {\bf{Input:}} $x \in [0, 1]$, privacy budget $\epsilon$, communication budget $b = \bout$, input bit-width $\bin$.
\STATE{\bf{Offline phase:}}
\STATE Let $\Bout = 2^b, \Bin = 2^{\bin}$.
\STATE Construct sampling probability matrix $P \in \mathbb{R}^{\Bin \times \Bout}$ and output alphabet $A = \{a_0,\ldots,a_{\Bout-1}\}$ to satisfy $\epsilon$-DP and unbiasedness constraints.
\STATE{\bf{Online phase:}}
\STATE $i = (\Bin-1) \cdot \Dither(x) \in \{0,1,\ldots,\Bin-1\}$.
\STATE Draw $j \in \{0,\ldots,\Bout-1\}$ from the categorical distribution defined by probability vector $P_i$.
\STATE {\textbf{Return}} $a_j$.
\end{algorithmic}
\end{algorithm}

Our privacy-aware compression mechanism operates in two phases.
In the offline phase, it selects an input bit-width value $\bin$ and pre-computes an output alphabet $A$ and a sampling probability matrix $P \in \mathbb{R}^{\Bin \times \Bout}$, where $\Bout = 2^b, \Bin = 2^{\bin}$. Both $P$ and $A$ are shared with the server and all clients. In the online phase, the client-side mechanism $\calM$ first uses dithering to round an input $x \in [0,1]$ to the grid $\{0, \frac{1}{\Bin-1}, \ldots, 1\}$ while maintaining unbiasedness, and then draws an index $j$ from the categorical distribution defined by the probability vector $P_i$, where $i = (\Bin-1) \cdot \Dither(x)$. The client then sends $a_j$ to the server. Algorithm \ref{alg:overview} summarizes the procedure in pseudo-code. Note that the strategy generalizes to any bounded input range by scaling $x$ appropriately.


In order for $\calM$ to satisfy $\epsilon$-DP and unbiasedness, we must impose the following constraints for the sampling probability matrix $P = [p_{i,j}]$ and output alphabet $A = \{a_j\}_{j=0}^{\Bout-1}$:
\vspace{-2ex}
\begin{subequations} \label{eq:constraints}
\begin{align}
    \text{Row-stochasticity:} & \quad \sum_{j=0}^{\Bout-1} p_{i,j} = 1 \quad \forall i \label{eq:row-stochastic} \\
    \text{Non-negativity:} & \quad p_{i,j} \ge 0 \quad \forall i,j \label{eq:non-negative} \\
    \text{$\epsilon$-DP:} & \quad p_{i',j} e^{-\epsilon} \le p_{i,j} \le p_{i', j} e^\epsilon \quad \forall i \ne i' \label{eq:dp-constraints} \\
    \text{Unbiasedness:} & \quad \sum_{j=0}^{\Bout-1} a_j p_{i,j} = \frac{i}{\Bin-1} \quad \forall i. \label{eq:unbiased}
\end{align}
\end{subequations}
Conditions \eqref{eq:row-stochastic} and \eqref{eq:non-negative} ensure that $P$ is a probability matrix. Condition \eqref{eq:dp-constraints} ensures $\epsilon$-DP, while condition \eqref{eq:unbiased} ensures unbiasedness. Note that these constraints only define the feasibility conditions for $P$ and $A$, and hence form the basis for a broad class of private mechanisms.
In the following sections, we show that two variants of an existing local DP mechanism Randomized Response~\cite{warner1965randomized} -- bit-wise Randomized Response and generalized Randomized Response -- can be realized as special cases of this family of mechanisms. 


\subsection{Unbiased Bitwise Randomized Response}

Randomized Response (RR)~\citep{warner1965randomized} is one of the simplest LDP mechanisms that sanitizes a single bit. Given a bit $y \in \{ 0, 1\}$, the RR mechanism outputs the $y$ with some probability $p$ and the flipped bit $1 - y$ with probability $1 - p$. If $p = \frac{1}{1 + e^{-\epsilon}}$, then the mechanism is $\epsilon$-local DP.

\smallskip\noindent\textbf{Unbiased Bitwise Randomized Response Mechanism.} The RR mechanism does not directly apply to our task as it is biased and applies to one bit. We obtain unbiasedness by using the output alphabet $A = \{ - \frac{1}{e^{\epsilon} - 1}, \frac{e^{\epsilon}}{e^{\epsilon} - 1} \}$, and repeat the one-bit mechanism $b$ times on each bit of $x$, with a privacy budget of $\epsilon/b$ each time.
It is not hard to see that unbiased RR with $b=1$ is a special case of Algorithm \ref{alg:overview}. For $b>1$, we can construct the resulting probability matrix $P$ by applying unbiased RR to each bit independently and similarly obtain the resulting output alphabet $A$.

We prove in Appendix \ref{sec:proofs} that Unbiased Bitwise Multiple RR satisfies $\epsilon$-local DP and is unbiased.

\subsection{Unbiased Generalized Randomized Response}

Generalized Randomized Response is a simple generalization of the one-bit RR mechanism for sanitizing a categorical value $x \in \{1, \ldots, K\}$. The mechanism transmits $x$ with some probability $p$, and a draw from a uniform distribution over $\{1, \ldots, K\}$ with probability $1 - p$. The mechanism satisfies $\epsilon$-local DP when $p = \frac{e^{\epsilon} - 1}{K + e^{\epsilon} - 1}$.

\smallskip\noindent\textbf{Unbiased Generalized Randomized Response.} We can adapt Generalized RR to our task by dithering the input $x$ to the grid $\{ 0, \frac{1}{\Bout-1}, \ldots, 1\}$ where $\Bout = 2^{\bout}$, and then transmitting the result using Generalized RR. Alternatively, we can derive the sampling probability matrix $P = \frac{e^{\epsilon} - 1}{\Bout + e^{\epsilon} - 1} I_{\Bout} + \frac{1}{\Bout + e^{\epsilon} - 1}$, where $I_{\Bout}$ is the identity matrix.
However, this leads to a biased output. To address this, we change the alphabet to $A = \{ a_0, a_1, \ldots, a_{\Bout-1}\}$ such that unbiasedness is maintained.
Specifically, for any $i \in \{ 0, \ldots, \Bout - 1 \}$, we need to ensure that when the input is $\frac{i}{\Bout-1}$, the expected output is also $\frac{i}{\Bout-1}$, which reduces to the following equation:
\begin{equation*} \label{eqn:airappor}
 a_i \cdot \frac{e^{\epsilon} - 1}{\Bout + e^{\epsilon} - 1} + \sum_{j = 0}^{\Bout-1} a_j \cdot \frac{1}{ \Bout + e^{\epsilon} - 1} = \frac{i}{\Bout-1}.
\end{equation*}
Writing this down for each $i$ gives $\Bout$ linear equations, solving which will give us the values of $a_0, \ldots, a_{\Bout-1}$. We establish the privacy and unbiasedness properties of Unbiased Generalized RR in Appendix \ref{sec:proofs}. A similar unbiased adaptation was also considered by \cite{balle2019privacy}. 

\begin{figure*}[t]
\centering
\includegraphics[width=\linewidth]{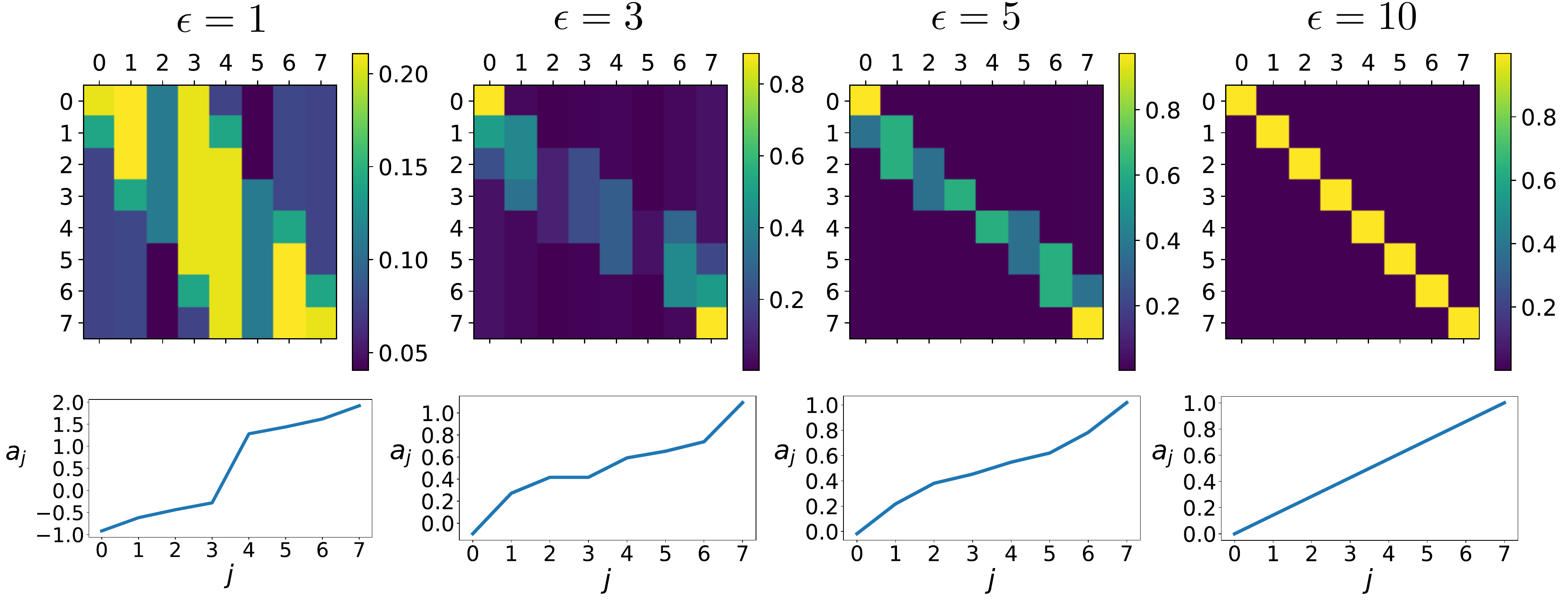}
\caption{Optimized sampling probability matrix $P$ (top row) and output alphabet $A = \{a_0,\ldots,a_{\Bout-1}\}$ (bottom row) of the MVU mechanism with $\bin = \bout = 3$ for $\epsilon=1,3,5,10$. At $\epsilon=1$, the DP constraint forces entries in each column to be similar, and the unbiasedness constraint causes the magnitude of $a_j$ to be large. At $\epsilon=10$, the weaker DP constraint allows the optimal $P$ matrix to become close to the identity matrix and $a_j \approx j/(B-1)$.}
\label{fig:p_and_alpha_samples}
\end{figure*}

\subsection{The MVU Mechanism}

A challenge with Unbiased Bitwise RR and Unbiased Generalized RR is that both algorithms are not intrinsically designed for ordinal or numerical values, which may result in poor accuracy upon aggregation. We next propose a new method that improves estimation accuracy by reducing the variance of each client's output while retaining unbiasedness and hence asymptotic consistency.

Our proposed method -- the \emph{Minimum Variance Unbiased} (MVU) mechanism --  addresses this problem by directly minimizing the variance of the client's output. This is done by solving the following optimization problem:
\begin{align} \label{eq:mvu_problem}
    \min_{\substack{p \in [0,1]^{\Bin \times \Bout} \\ a \in \mathbb{R}^{\Bout}}} & \quad \sum_{i=0}^{\Bin-1} \sum_{j=0}^{\Bout-1} p_{i,j} \left(\frac{i}{\Bin-1} - a_j\right)^2 \\
    \text{subject to} & \quad \text{Conditions } \eqref{eq:row-stochastic}-\eqref{eq:unbiased}. \nonumber
\end{align}
The objective in \eqref{eq:mvu_problem} measures the variance of the output of the mechanism when the input $i$ is uniformly distributed over the set $\{0, \frac{1}{\Bin-1}, \dots, 1\}$. Conditions \eqref{eq:row-stochastic}-\eqref{eq:unbiased} ensure that the MVU mechanism is $\epsilon$-DP and unbiased, hence satisfying requirements for our task.

\smallskip\noindent\textbf{Solving the MVU mechanism design problem.} We solve \eqref{eq:mvu_problem} using one of two approaches depending on size of the probability matrix $P$ and $\epsilon$. For smaller problems and when $\epsilon$ is not too small, we use a trust region interior-point solver~\citep{conn2000trust}. As $\epsilon$ approaches $0$, the problem becomes poorly conditioned and we only approximately solve the problem by relaxing the unbiasedness constraint~\eqref{eq:unbiased}. In this case we use an alternating minimization heuristic where we alternate between fixing the values $a_j$ and solving for $p_{i,j}$, and holding $p_{i,j}$ fixed and solving for $a_j$, while incorporating constraint~\eqref{eq:unbiased} as a soft penalty in the objective. Each of the corresponding subproblems is a quadratic program and can be solved efficiently. Figure~\ref{fig:p_and_alpha_samples} shows examples of the MVU mechanism for $\bin = \bout = 3$ and $\epsilon \in \{1, 3, 5, 10\}$ obtained using the trust region solver.

\smallskip\noindent\textbf{Relationship between DP and compression.} The MVU mechanism highlights an intriguing connection between DP and compression. Since the mechanism hides information in the input $x$ by perturbing it with random noise, as $\epsilon \rightarrow 0$, fewer bits are required to describe the noisy output $\calM(x)$. In the limiting case of $\epsilon=0$, all information is lost and the output can be described by zero bits. In Appendix \ref{sec:experiment_details}, we demonstrate this argument concretely by showing that as $\epsilon \rightarrow 0$, the marginal benefit of having a larger communication budget decreases. 

\section{Extensions}
\label{sec:extension}

We now show how to extend the MVU mechanism to obtain privacy-aware and accurate compression mechanisms for metric-DP and vector spaces.

\subsection{Metric DP}

In location privacy, client devices send their obscured locations to a central server for aggregation. Metric DP (Definition~\ref{def:metricdp}) is a variation of LDP that applies to this use-case. We are given a position $x$ and a metric $d$ which measures how far apart two positions are. Our goal is to output a private position $x'$ so that fine-grained properties of $x$ (such as, exact address, city block) are hidden, while coarse-grained properties (such as, city, or zip-code) are preserved. 

We show how to adapt the MVU mechanism to metric DP. For simplicity, suppose that we measure position on the line, so $x \in [0, 1]$.  We modify Condition~\eqref{eq:dp-constraints} to instead satisfy the metric DP constraint with respect to the metric $d$:
\begin{equation} \label{eqn:qpmetricdp}
	p_{i', j} e^{-\epsilon d(i/(\Bin-1), i'/(\Bin-1))} \leq p_{ i, j} \leq p_{ i', j} e^{\epsilon d( i/(\Bin-1), i'/(\Bin-1))}.
\end{equation}
Thus we can get an MVU mechanism for metric DP by solving the modified optimization problem in \eqref{eq:mvu_problem} and following the same procedure in Algorithm~\ref{alg:overview}.

\subsection{Extension to Vector Spaces}

We next look at extending the MVU mechanism to vector spaces. Specifically, a client now holds a $d$-dimensional vector $\bx$ in a domain $\calX \subseteq \mathbb{R}^d$, and its goal is to output an $\epsilon$-local DP version that can be communicated in $bd$ bits. The domain $\calX$ is typically a unit $L_p$-norm ball for $p \geq 1$.


A plausible approach is to apply the scalar MVU mechanism independently for each coordinate of $\bx$. While this will provide the optimal accuracy for $p = \infty$, for $p < \infty$, the client's variance will be higher. A second approach is to extend the MVU mechanism directly to $\calX$ by using an alphabet $A \times A \times \ldots \times A = A^d$ and then solving the corresponding optimization problem~\eqref{eq:mvu_problem}. Unfortunately this is computationally intractable even for moderate $d$.


Instead, we show how to obtain a more computationally tractable approximation when $\calX$ is an $L_p$-ball. We are motivated by the following lemma.


\begin{lemma} \label{lem:metrictovector}
Let $\calX$ be the unit $L_p$-ball with diameter $\Delta$. Suppose $\calM$ is an $\epsilon$-metric DP scalar mechanism with $d(y, y') = |y - y'|^p$. Then, the mechanism $\calM_d: \calX \rightarrow \mathbb{R}^d$ that maps $\bx$ to the vector $(\calM(\bx_1), \ldots, \calM(\bx_d))$ is $\epsilon \Delta^p$-local DP. Additionally, if $\calM$ is unbiased, then $\calM_d$ is unbiased as well.
\end{lemma}

Lemma \ref{lem:metrictovector} suggests the following algorithm: Use the MVU mechanism for $\epsilon$-metric DP with $d(y, y') = |y - y'|^p$ for each coordinate, then combine to get an $\epsilon$-local DP solution for vectors with $L_p$-sensitivity $\Delta$. Since $\| \cdot \|_\infty \leq \| \cdot \|_p$, each coordinate of $\bx$ lies in a bounded range $[-\Delta,\Delta]$, so we can scale $\bx$ by $\bx' \leftarrow (\bx + \Delta) / 2\Delta$ so that all entries belong to $[0,1]$ and the MVU mechanism can be applied to $\bx'$. Note that this scaling operation changes the $L_p$-sensitivity to $1/2$.

This solution is computationally tractable since we only need to solve an optimization problem for the scalar MVU mechanism -- so involving $\approx \Bout^2 = 2^{2\bout} $ variables and constraints (instead of $\approx 2^{2\bout d}$). We investigate how this mechanism works in practice in Section~\ref{sec:experiments}.

\subsection{Composition using R\'{e}nyi-DP}
\label{sec:composition}


Repeated applications of the MVU mechanism will give an additive sequential privacy composition guarantee as in standard $\epsilon$-DP. We next show how to get tighter composition bounds for the MVU mechanism using RDP accounting as in~\cite{mironov2017renyi}.

Suppose that $\bx, \bx' \in \{0,1/(\Bin-1),\ldots,1\}^d$ are quantized $d$-dimensional vectors, and let $Q_0, Q_1$ be the output distributions of the mechanism $\calM$ for inputs $\bx, \bx'$, respectively. By the definition of R\'{e}nyi divergence~\citep{renyi1961measures},
\begin{equation*}
D_\alpha(Q_0 || Q_1)
= \frac{1}{\alpha-1} \sum_{l=1}^d \log \sum_{j=0}^{\Bin-1} \frac{p_{\bi_l,j}^\alpha}{p_{\bi_l',j}^{\alpha-1}},
\end{equation*}
where $\bi, \bi' \in \{0,1,\ldots,\Bin-1\}^d$ are such that $\bx = \bi / (\Bin-1)$ and $\bx' = \bi' / (\Bin-1)$. Let $D^\alpha$ denote the $\Bin \times \Bin$ matrix with entries $D^\alpha_{i,i'} = \frac{1}{\alpha-1} \log \sum_{j=0}^{\Bin-1} p_{i,j}^\alpha / p_{i',j}^{\alpha - 1}$. Then, computation of the $\alpha$-RDP parameter for $\calM$ can be formulated as the following combinatorial optimization problem:
\begin{equation*}
\label{eq:comb_opt}
\max_{\bi, \bi' \in \{0,1,\ldots,\Bin-1\}^d} \: \sum_{l=1}^d D^\alpha_{\bi_l, \bi_l'}
\quad \text{s.t. } \: \| \bi - \bi' \|_p^p \leq (\Bin-1)^p \Delta^p.
\end{equation*}
This optimization problem is in fact an instance of the \emph{multiple-choice knapsack problem}~\citep{sinha1979multiple} and admits an efficient linear program relaxation by converting the integer vectors $\bi, \bi'$ to probability vectors, \emph{i.e.},
\begin{align}
\label{eq:comb_opt_lp}
\max_{\bp \in \mathbb{R}^{d \times \Bin \times \Bin}} &\quad \sum_{l=1}^d \langle D^\alpha, \bp_l \rangle_F \\
\text{subject to } &\quad \sum_{l=1}^d \langle C, \bp_l \rangle_F \leq (\Bin-1)^p \Delta^p \nonumber \\
&\quad \sum_{i,j} (\bp_l)_{ij} \leq 1 \text{ and } \bp_l \geq 0 \; \forall l, \nonumber
\end{align}
where $\langle \cdot, \cdot \rangle_F$ denotes Frobenius (vectorized) inner product and $C$ denotes the distance matrix with entries $C_{ij} = (i - j)^p$.
This LP relaxation can still be prohibitively expensive to solve for large $d$ since $\bp$ contains $d\Bin^2$ variables. Fortunately, in such cases, we can obtain an upper bound via the greedy solution; see Appendix \ref{sec:proofs} for the proof.

\begin{lemma}\label{lem:greedy}
Let $(i^*, j^*) = \argmax_{i,j} D^\alpha_{ij} / C_{ij}$ and let $d_0 = (\Bin-1)^p \Delta^p / C_{i^* j^*}$. Then \eqref{eq:comb_opt_lp} $\leq d_0 D_{i^* j^*}^\alpha$.
\end{lemma}


To summarize, for composition with RDP accounting at order $\alpha$, we can either solve the LP relaxation in \eqref{eq:comb_opt_lp} or compute the greedy solution to obtain an upper bound for $D_\alpha(P || Q)$, and then apply the usual composition for RDP.

\begin{figure*}[t!]
\centering
\includegraphics[width=\linewidth]{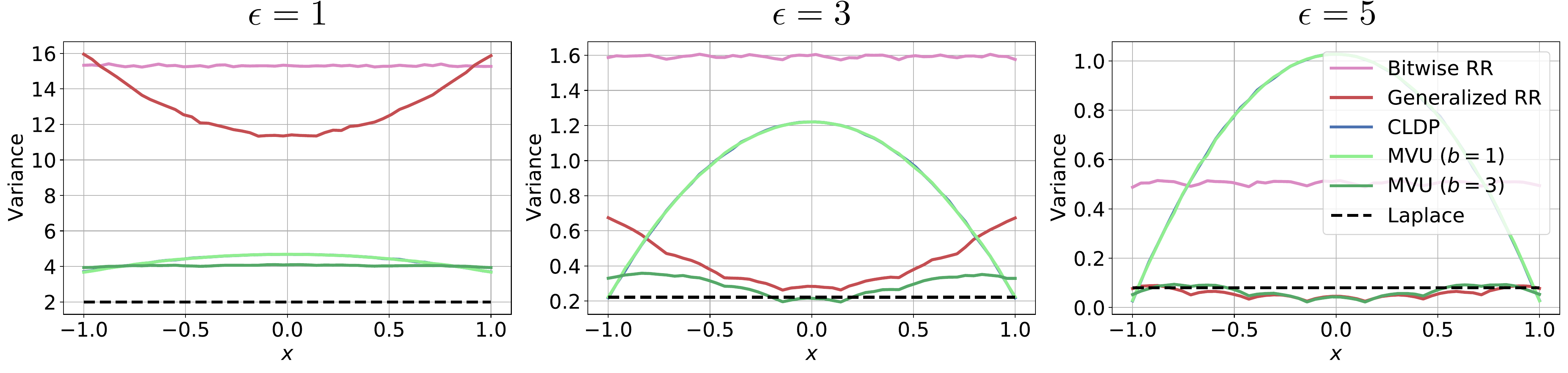}
\caption{Distributed mean estimation for scalar data with LDP $\epsilon=1,3,5$. The MVU mechanism with budget $b=1$ recovers the CLDP mechanism and the two curves coincide, while with $b=3$ MVU attains a low variance across all input values compared to the baseline mechanisms. See text for details.
}
\label{fig:dme_scalar_comparison}
\end{figure*}

\section{Experiments}
\label{sec:experiments}


We evaluate the MVU mechanism on two sets of experiments: Distributed mean estimation and federated learning. Our goal is to demonstrate that MVU can attain a better privacy-utility trade-off at low communication budgets compared to other private compression mechanisms.
Code to reproduce our results can be found in the repo \url{https://github.com/facebookresearch/dp_compression}.

\subsection{Distributed mean estimation}
\label{sec:dme}

In distributed mean estimation (DME), a set of $n$ clients each holds a private vector $\bx_i \in \mathbb{R}^d$, and the server would like to privately estimate the mean $\bar{\bx} = \frac{1}{n} \sum_{i=1}^n \bx_i$.

\paragraph{Scalar DME.} We first consider the setting of scalar data, \emph{i.e.}, $d=1$. For a fixed value $x \in [-1,1]$, we set $\bx_i = x$ for all $i=1,\ldots,n$ with $n=100,000$ and then privatize them before taking average. We measure the squared difference between the private estimate and $\bar{\bx} = x$, which is coincidentally the variance of the mechanism at $x$. The baseline mechanisms that we evaluate against are (unbiased) Bitwise Randomized Response (bRR), (unbiased) Generalized Randomized Response (gRR), the communication-limited local differentially private (CLDP) mechanism~\citep{girgis2021shuffled}, and the Laplace mechanism without any compression. The CLDP mechanism uses a fixed communication budget of $b=1$, whereas for bRR and gRR we set $b=3$, and for MVU we set $b=1,3$. 

Figure \ref{fig:dme_scalar_comparison} shows the plot of input value $x$ vs. variance of the private mechanism at $x$.
Interestingly, MVU with $b=1$ recovers the CLDP mechanism for $\epsilon=1,3,5$, while MVU with $b=3$ is consistently the lowest variance private compression mechanism. For larger $\epsilon$, it is evident that the variance of both gRR and MVU are comparable or even slightly lower that of the Laplace mechanism, even when compressing to only $b=3$ bits in their output.

\begin{figure*}[t]
    \centering
    \begin{subfigure}{.49\textwidth}
      \centering
      \includegraphics[width=\linewidth]{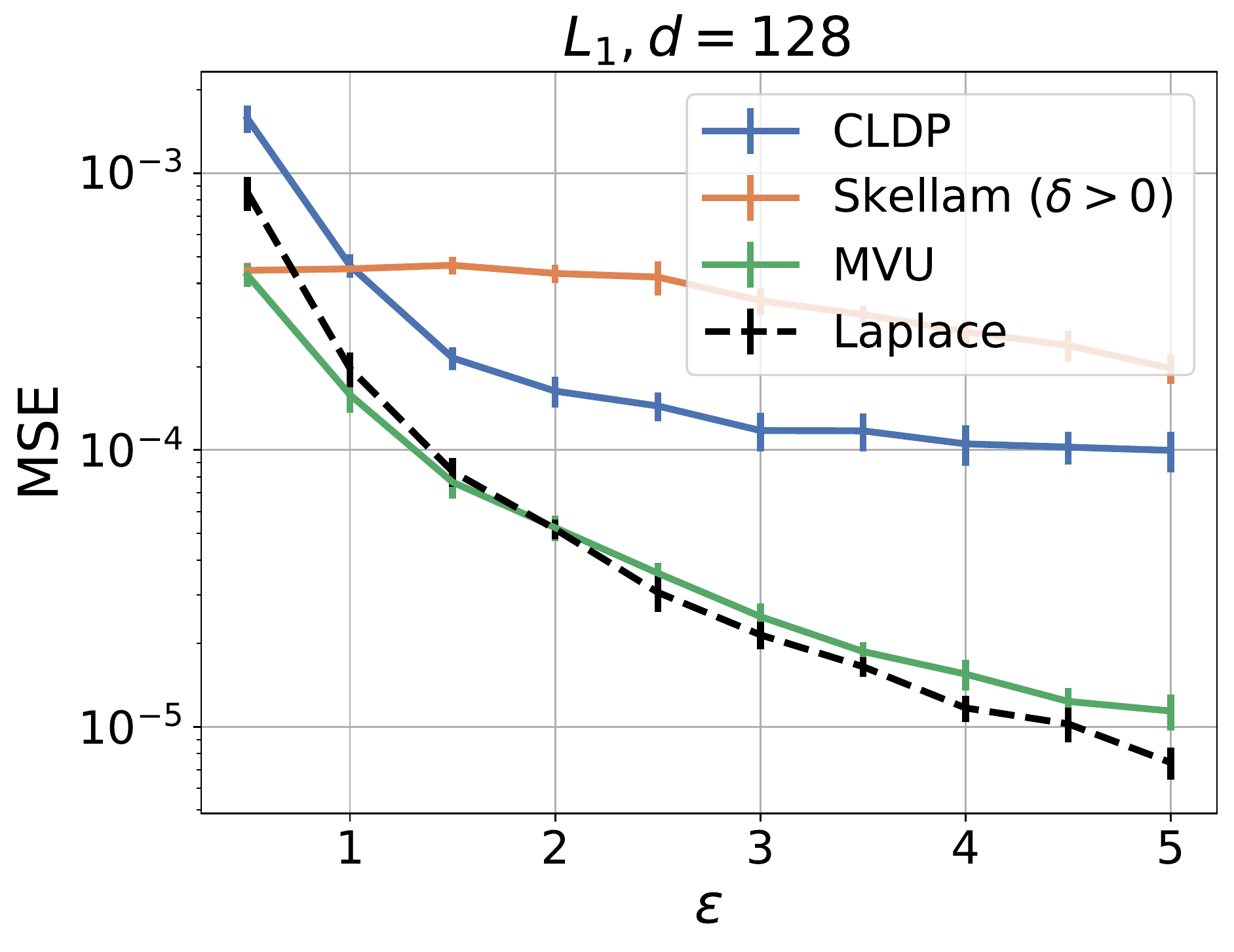}
    \end{subfigure}
    \begin{subfigure}{.49\textwidth}
      \centering
      \includegraphics[width=\linewidth]{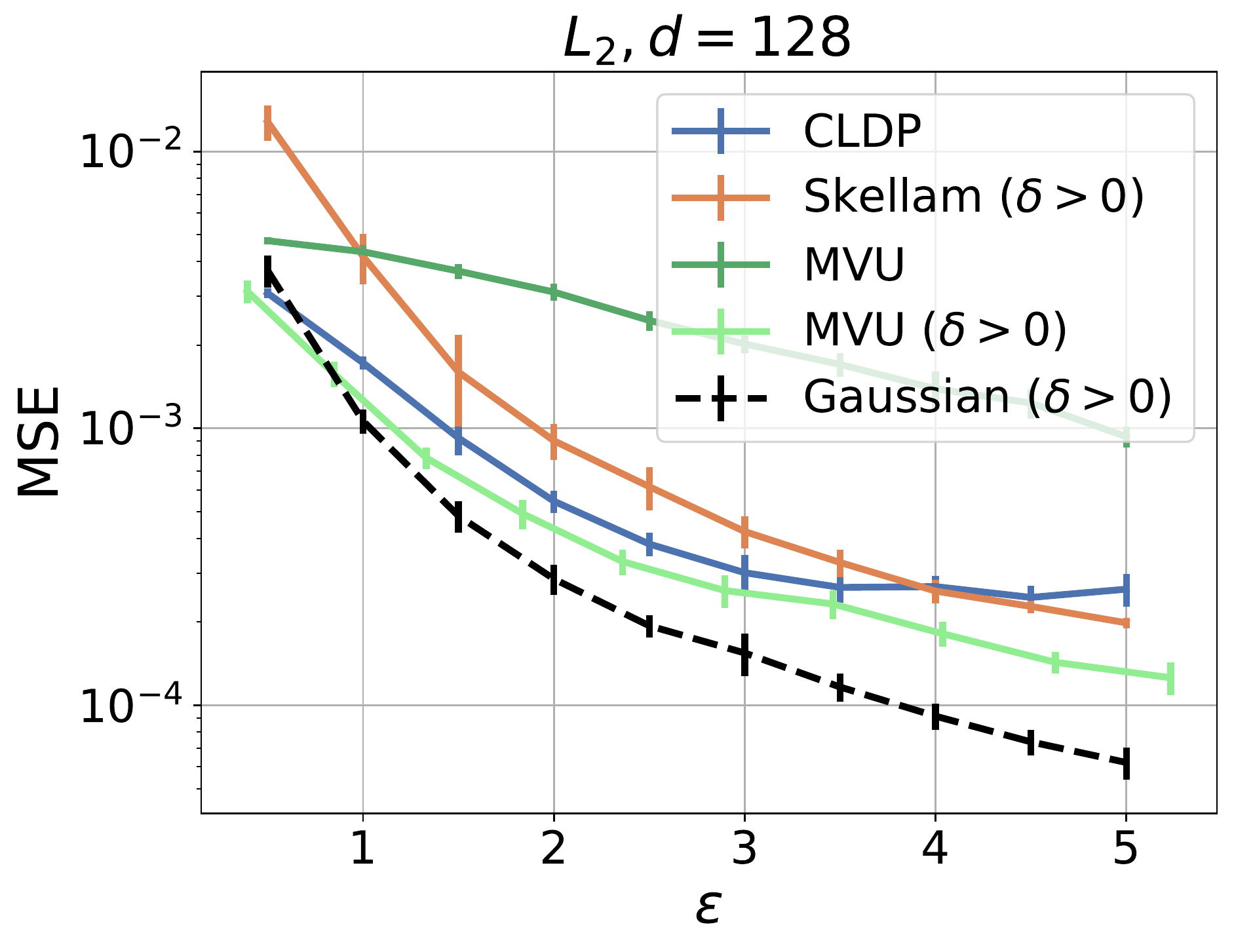}
      \end{subfigure}
    \caption{Distributed mean estimation for $n=10,000$ data vectors with $L_1$- (left) and $L_2$-sensitivity (right). Error bars represent standard deviation across $10$ repeated runs with different private vectors. Methods that are $(\epsilon,\delta)$-DP use the same value of $\delta=1/(n+1)$. The MVU mechanism can attain an MSE close to that of the Laplace and Gaussian mechanisms while compressing the output to only $b=3$ bits per coordinate. }
    \label{fig:dme_comparison}
\end{figure*}

\paragraph{Vector DME.} We next look at vector data with $d=128$ and $n=10,000$. We draw the sensitive vectors from two distinct distributions\footnote{We intentionally avoided zero-mean distributions since some of the private mechanisms converge to the all-zero vector as $\epsilon \rightarrow 0$.}: (i) Uniform at random from $[0,1]^d$ and then normalize to $L_1$-norm of 1; and (ii) Uniform over the spherical sector $\mathbb{S}^{d-1} \cap \mathbb{R}_{\geq 0}^d$. In these settings, the vectors $\bx_i$ have $L_1$- and $L_2$-sensitivity of 1, respectively.

For baselines, we consider the CLDP mechanism~\citep{girgis2021shuffled}, the Skellam mechanism~\citep{agarwal2021skellam}, the Laplace mechanism (for setting (i) only), and the Gaussian mechanism (for setting (ii) only). Both the Skellam and the Gaussian mechanisms are $(\epsilon,\delta)$-DP for $\delta > 0$. For a given $\epsilon>0$, we set $\delta = 1/(n+1)$ and choose the noise parameter $\mu$ for the Skellam mechanism using the optimal RDP conversion, and the noise parameter $\sigma$ for the Gaussian mechanism using the analytical conversion in \citet{balle2018improving}.
For communication budget, we set $b=3$ for MVU and $b=16$ for Skellam (which requires a large $b$ in order to prevent truncation error). The CLDP mechanism does not allow flexible selection of communication budget, and instead outputs a \emph{total} number of $\log_2(d) + 1$ bits for the $L_1$-sensitivity setting, and $b = \log_2(d) + 1 = 8$ bits \emph{per coordinate} for the $L_2$-sensitivity setting. See Appendix \ref{sec:experiment_details} for a more detailed explanation.

Figure \ref{fig:dme_comparison} shows the mean squared error (MSE) for privately estimating $\bar{\bx}$ across different values of $\epsilon$. In the left plot corresponding to the $L_1$-sensitivity setting, MVU can attain MSE close to the Laplace mechanism at a greatly reduced $b=3$ bits per coordinate. In comparison, CLDP and Skellam attain MSE that is more than an order of magnitude higher than Laplace.

The right plot corresponds to $L_2$-sensitivity. Here, the MVU mechanism (dark green line) is significantly less competitive than the baselines. This is because the $L_2$-metric DP constraint for the MVU mechanism forces rows of the sampling probability matrix $P$ to be near-identical, hence is near-singular and does not admit a well-conditioned unbiased solution.
To address this problem, we instead optimize the MVU mechanism to satisfy $L_1$-metric DP and use the R\'{e}nyi accounting in Section \ref{sec:composition} to compute its RDP guarantee, then apply RDP-to-DP conversion to give an $(\epsilon,\delta)$-DP guarantee at $\delta=\frac{1}{n+1}$. The light green line shows the performance of the $L_1$-metric DP mechanism, which now slightly outperforms both CLDP and Skellam at a much lower communication budget of $b=3$. These results demonstrate that the MVU mechanism attains better utility vs. compression trade-off for vector data as well.

\begin{figure*}[t]
    \centering
    \begin{subfigure}{.49\textwidth}
      \centering
      \includegraphics[width=\linewidth]{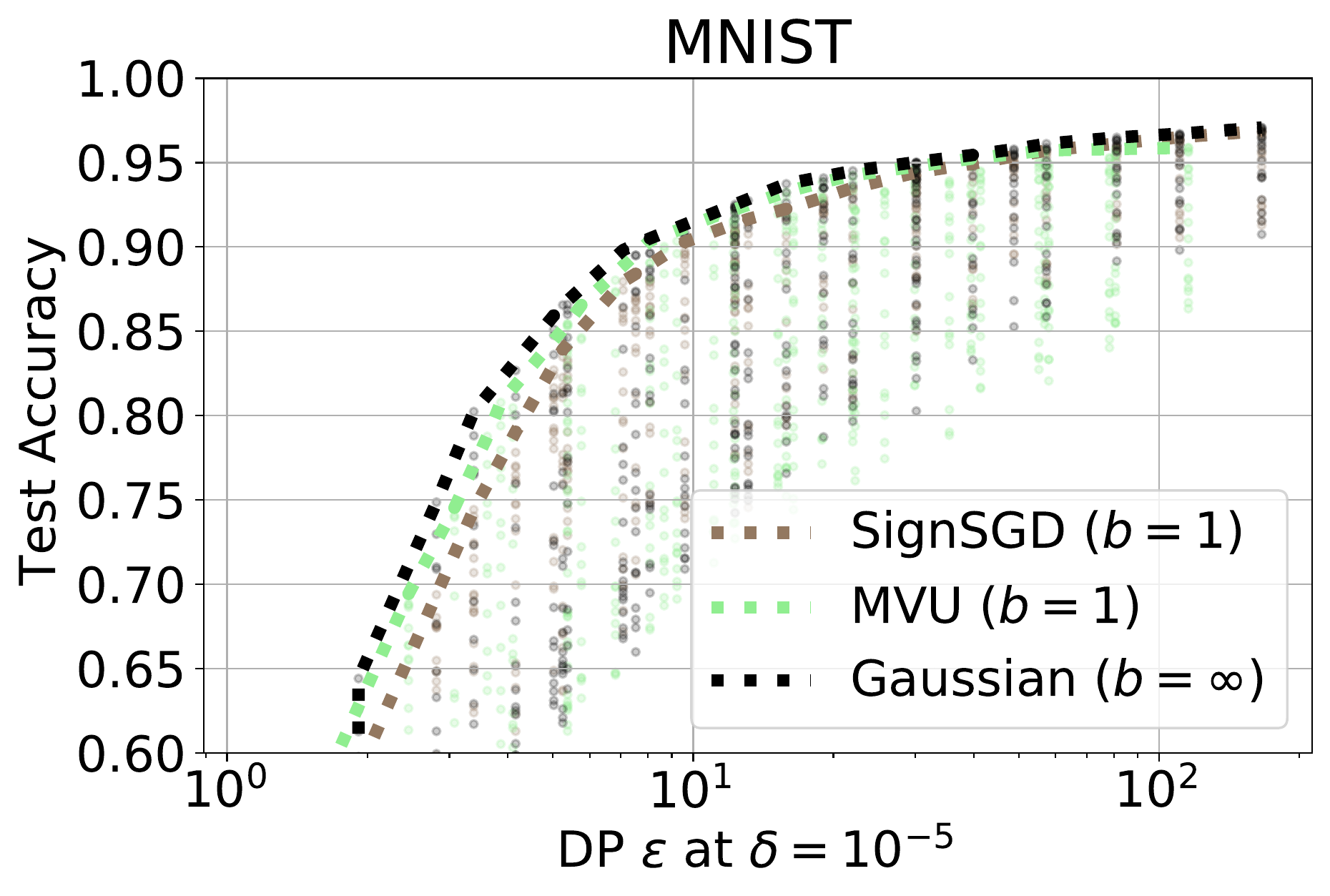}
    \end{subfigure}
    \begin{subfigure}{.49\textwidth}
      \centering
      \includegraphics[width=\linewidth]{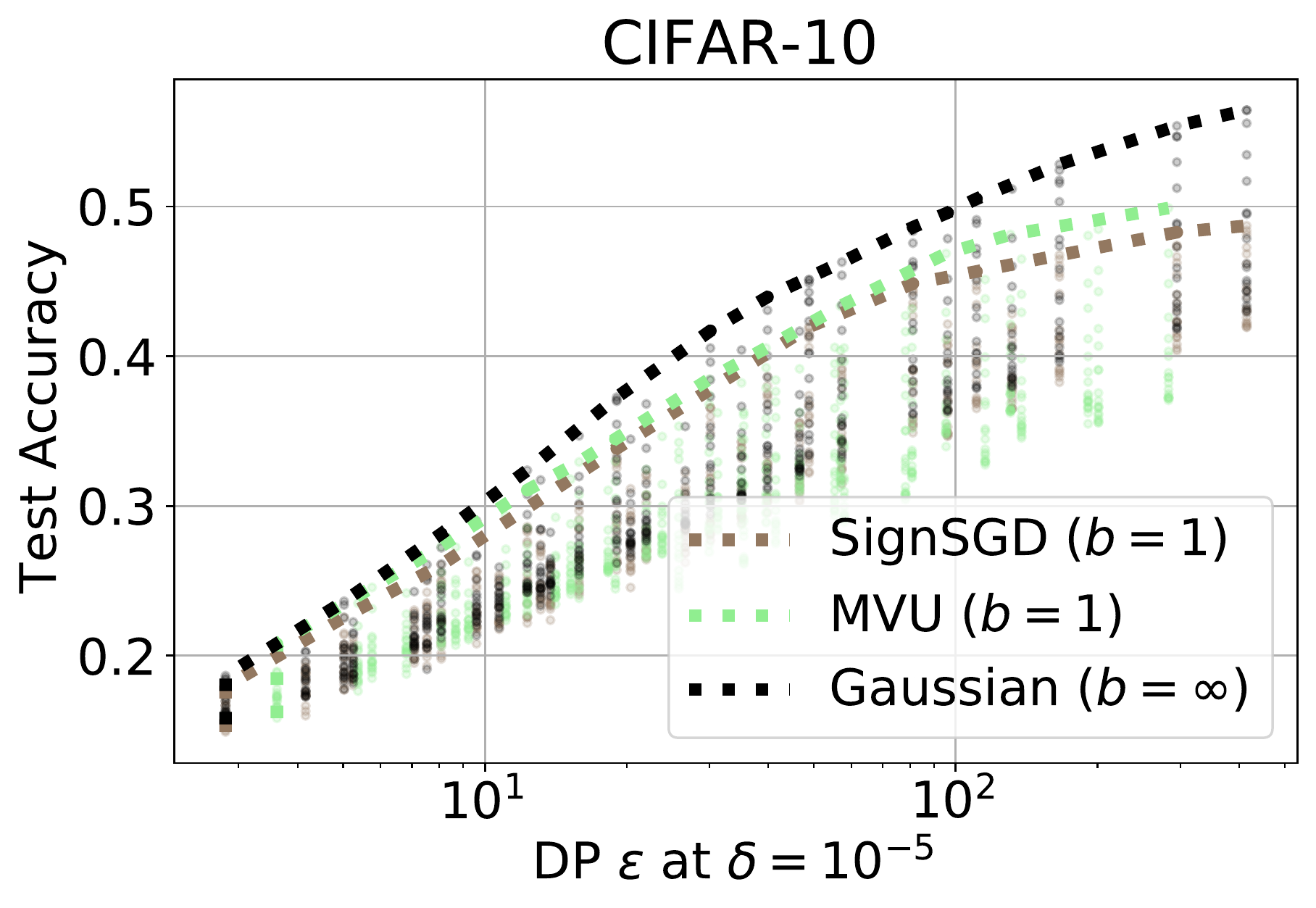}
      \end{subfigure}
\caption{DP-SGD training with Gaussian mechanism, stochastic signSGD and MVU mechanism on MNIST (left) and CIFAR-10 (right). Each point corresponds to a single hyperparameter setting, and dashed line shows Pareto frontier of privacy-utility trade-off. MVU mechanism outperforms signSGD at the same communication budget of $b=1$.}
\label{fig:dpsgd_comparison}
\end{figure*}

\subsection{Private SGD}
\label{sec:fl}

Federated learning~\citep{mcmahan2017communication} often employs DP to protect the privacy of the clients' updates. We next evaluate the MVU mechanism for this use case and show that it can serve as a drop-in replacement for the Gaussian mechanism for FL protocols, providing similar DP guarantees for the client update while reducing communication.

In detail, for MNIST and CIFAR-10~\citep{krizhevsky2009learning}, we train a linear classifier on top of features extracted by a scattering network~\citep{oyallon2015deep} similar to the one used in \citet{tramer2020differentially}; see Appendix \ref{sec:experiment_details} for details.
The base private learning algorithm is DP-SGD with Gaussian gradient perturbation~\citep{abadi2016deep} and R\'{e}nyi-DP accounting. The private compression baselines are the MVU mechanism with budget $b=1$ and stochastic signSGD~\citep{jin2020stochastic} -- a specialized private gradient compression scheme for federated SGD that applies the Gaussian mechanism and outputs its coordinate-wise sign. Similar to the distributed mean estimation experiment with $L_2$-sensitivity, we optimize the MVU mechanism to satisfy $L_1$-metric DP and then compute its R\'{e}nyi privacy guarantee as in Section \ref{sec:composition}.

Figure \ref{fig:dpsgd_comparison} shows the privacy-utility trade-off curves. We sweep over a grid of hyperparameters (see Appendix \ref{sec:experiment_details} for details) for each mechanism and plot the resulting $\epsilon$ and test accuracy as a point in the scatter plot. The dashed line is the Pareto frontier of optimal privacy-utility trade-off. The result shows that MVU mechanism outperforms signSGD---a specially-designed gradient compression mechanism for federated learning---at nearly all privacy budgets with the same communication cost of \emph{one bit per coordinate}. We include an additional result for a small convolutional network in Appendix \ref{sec:experiment_details}, where we observe similar findings.

\section{Related Work}
\label{sec:relwork}



Federated data analysis with local DP is now a standard solution for analyzing sensitive data held by many user devices. A body of work~\citep{erlingsson2014rappor, kairouz2016discrete, acharya2019hadamard} provides methods for analytics over categorical data. The main methods here are Randomized Response~\citep{warner1965randomized}, RAPPOR~\citep{erlingsson2014rappor} and the Hadamard Mechanism~\citep{acharya2019hadamard}. \citet{chen2020breaking} shows that the Hadamard Mechanism uses near-optimal communication for categorical data. 


In work on federated statistics or learning for real-valued data, \cite{Cormode2021bit} provides asymptotically consistent algorithms for transmitting scalars.  
They propose to first sample one or a subset of indices of bits in the fixed-point representation of the input, and then apply randomized response independently to each of these bits. \cite{girgis2020shuffled} provides mechanisms for distributed mean estimation from vectors inside unit $L_p$ balls. Unlike our method, which provides a near-optimal solution under any given communication budget, their methods use specific communication budgets and are not readily generalizable to any budget $b$. Finally, \citet{Amiri2021compressive} propose to obtain a quantized DP mechanism by composing subtractive dithering with the Gaussian mechanism, and doing privacy accounting that factors in both. In contrast, we simply use (non-subtractive) dithering to initially obtain a fixed-point representation, and then design a mechanism to  quantize and provide DP.



A large body of work focuses on federated optimization methods with compressed communication~\citep{konevcny2016federated,horvath2019stochastic,das2020faster,haddadpour2021federated,gorbunov21marina}. While most propose biased compression methods (e.g., top-$k$ sparsification), such approaches require the use of error feedback to avoid compounding errors~\citep{seide2014one,stich2020error}. However, error feedback is inherently incompatible with DP~\citep{jin2020stochastic}, unlike our MVU mechanism. 
\section{Conclusion and Limitations}

We introduce the MVU framework to jointly design scalar compression and DP mechanisms, and extend it to the vector and metric-DP settings. We show that the MVU mechanism attains a better utility-compression trade-off for both scalar and vector mean estimation compared to other approaches in the literature. Our work shows that co-designing the compression and privacy-preserving components can lead to more efficient differentially private mechanisms for federated data analysis.

\paragraph{Limitations.} Our work presents several opportunities for further improvement. 1. For vector dithering, Appendix \ref{sec:experiment_details} shows that the input vector's norm can increase by a small additive factor. Our current solution of conditional random dithering introduces a small but non-negligible bias. Future work on unbiased norm-preserving vector dithering may be able to alleviate this issue. 2. Optimizing the MVU mechanism for large values of the input/output bit width $\bin$ and $\bout$ can be prohibitively expensive, even with the alternating minimization heuristic. In order to scale the solution to higher-dimensional vectors, further effort in designing more efficient solutions for the MVU mechanism may be needed. 3. While our work focuses on local differential privacy, it may be possible to combine our approach with secure aggregation protocols to derive central differential privacy guarantees. However, since the MVU mechanism is not additive, further analysis is required to characterize the distribution of the aggregate for our mechanism, which we leave for future work.





\section*{Acknowledgements}
We thank Graham Cormode, Huanyu Zhang, and anonymous reviewers for insightful comments and suggestions that helped shape our final draft.

\bibliography{mike_refs,privacy}

\begin{thebibliography}{43}
\providecommand{\natexlab}[1]{#1}
\providecommand{\url}[1]{\texttt{#1}}
\expandafter\ifx\csname urlstyle\endcsname\relax
  \providecommand{\doi}[1]{doi: #1}\else
  \providecommand{\doi}{doi: \begingroup \urlstyle{rm}\Url}\fi

\bibitem[Abadi et~al.(2016)Abadi, Chu, Goodfellow, McMahan, Mironov, Talwar,
  and Zhang]{abadi2016deep}
Martin Abadi, Andy Chu, Ian Goodfellow, H~Brendan McMahan, Ilya Mironov, Kunal
  Talwar, and Li~Zhang.
\newblock Deep learning with differential privacy.
\newblock In \emph{Proceedings of the 2016 ACM SIGSAC conference on computer
  and communications security}, pages 308--318, 2016.

\bibitem[Acharya et~al.(2019)Acharya, Sun, and Zhang]{acharya2019hadamard}
Jayadev Acharya, Ziteng Sun, and Huanyu Zhang.
\newblock Hadamard response: Estimating distributions privately, efficiently,
  and with little communication.
\newblock In \emph{The 22nd International Conference on Artificial Intelligence
  and Statistics}, pages 1120--1129. PMLR, 2019.

\bibitem[Agarwal et~al.(2021)Agarwal, Kairouz, and Liu]{agarwal2021skellam}
Naman Agarwal, Peter Kairouz, and Ziyu Liu.
\newblock The skellam mechanism for differentially private federated learning.
\newblock \emph{Advances in Neural Information Processing Systems}, 34, 2021.

\bibitem[Alistarh et~al.(2017)Alistarh, Grubic, Li, Tomioka, and
  Vojnovic]{Alistarh2017qsgd}
Dan Alistarh, Demjan Grubic, Jerry Li, Ryota Tomioka, and Milan Vojnovic.
\newblock {QSGD}: Communication-efficient {SGD} via gradient quantization and
  encoding.
\newblock In \emph{Advances in Neural Information Processing Systems
  (NeurIPS)}, 2017.

\bibitem[Amiri et~al.(2021)Amiri, Belloum, Klous, and
  Gommans]{Amiri2021compressive}
Saba Amiri, Adam Belloum, Sander Klous, and Leon Gommans.
\newblock Compressive differentially-private federated learning through
  universal vector quantization.
\newblock In \emph{AAAI Workshop on Privacy-Preserving Artificial
  Intelligence}, 2021.

\bibitem[Andr{\'e}s et~al.(2013)Andr{\'e}s, Bordenabe, Chatzikokolakis, and
  Palamidessi]{andres2013geo}
Miguel~E Andr{\'e}s, Nicol{\'a}s~E Bordenabe, Konstantinos Chatzikokolakis, and
  Catuscia Palamidessi.
\newblock Geo-indistinguishability: Differential privacy for location-based
  systems.
\newblock In \emph{Proceedings of the 2013 ACM SIGSAC conference on Computer \&
  communications security}, pages 901--914, 2013.

\bibitem[Aysal et~al.(2008)Aysal, Coates, and Rabbat]{Aysal2008distributed}
T.~Can Aysal, Mark~J. Coates, and Michael~G. Rabbat.
\newblock Distributed average consensus with dithered quantization.
\newblock \emph{IEEE Trans. Signal Processing}, 56\penalty0 (10):\penalty0
  4905--4918, Oct. 2008.

\bibitem[Balle and Wang(2018)]{balle2018improving}
Borja Balle and Yu-Xiang Wang.
\newblock Improving the gaussian mechanism for differential privacy: Analytical
  calibration and optimal denoising.
\newblock In \emph{International Conference on Machine Learning}, pages
  394--403. PMLR, 2018.

\bibitem[Balle et~al.(2019)Balle, Bell, Gasc{\'o}n, and
  Nissim]{balle2019privacy}
Borja Balle, James Bell, Adri{\`a} Gasc{\'o}n, and Kobbi Nissim.
\newblock The privacy blanket of the shuffle model.
\newblock In \emph{Annual International Cryptology Conference}, pages 638--667.
  Springer, 2019.

\bibitem[Canonne et~al.(2020)Canonne, Kamath, and Steinke]{canonne2020discrete}
Cl{\'e}ment~L Canonne, Gautam Kamath, and Thomas Steinke.
\newblock The discrete gaussian for differential privacy.
\newblock \emph{Advances in Neural Information Processing Systems},
  33:\penalty0 15676--15688, 2020.

\bibitem[Chatzikokolakis et~al.(2013)Chatzikokolakis, Andr{\'e}s, Bordenabe,
  and Palamidessi]{chatzikokolakis2013broadening}
Konstantinos Chatzikokolakis, Miguel~E Andr{\'e}s, Nicol{\'a}s~Emilio
  Bordenabe, and Catuscia Palamidessi.
\newblock Broadening the scope of differential privacy using metrics.
\newblock In \emph{International Symposium on Privacy Enhancing Technologies
  Symposium}, pages 82--102. Springer, 2013.

\bibitem[Chen et~al.(2020)Chen, Kairouz, and Ozgur]{chen2020breaking}
Wei-Ning Chen, Peter Kairouz, and Ayfer Ozgur.
\newblock Breaking the communication-privacy-accuracy trilemma.
\newblock \emph{Advances in Neural Information Processing Systems},
  33:\penalty0 3312--3324, 2020.

\bibitem[Conn et~al.(2000)Conn, Gould, and Toint]{conn2000trust}
Andrew~R. Conn, Nicholas I.~M. Gould, and Philippe~L. Toint.
\newblock \emph{Trust Region Methods}.
\newblock SIAM, 2000.

\bibitem[Cormode and Markov(2021)]{Cormode2021bit}
Graham Cormode and Igor~L. Markov.
\newblock Bit-efficient numerical aggregation and stronger privacy for trust in
  federated analytics.
\newblock \emph{arXiv preprint arXiv:2108.01521}, Aug. 2021.

\bibitem[Das et~al.(2020)Das, Acharya, Hashemi, Sanghavi, Dhillon, and
  Topcu]{das2020faster}
Rudrajit Das, Anish Acharya, Abolfazl Hashemi, Sujay Sanghavi, Inderjit~S
  Dhillon, and Ufuk Topcu.
\newblock Faster non-convex federated learning via global and local momentum.
\newblock \emph{arXiv preprint arXiv:2012.04061}, 2020.

\bibitem[Duchi et~al.(2013)Duchi, Jordan, and Wainwright]{duchi2013local}
John~C Duchi, Michael~I Jordan, and Martin~J Wainwright.
\newblock Local privacy and statistical minimax rates.
\newblock In \emph{2013 IEEE 54th Annual Symposium on Foundations of Computer
  Science}, pages 429--438. IEEE, 2013.

\bibitem[Erlingsson et~al.(2014)Erlingsson, Pihur, and
  Korolova]{erlingsson2014rappor}
{\'U}lfar Erlingsson, Vasyl Pihur, and Aleksandra Korolova.
\newblock Rappor: Randomized aggregatable privacy-preserving ordinal response.
\newblock In \emph{Proceedings of the 2014 ACM SIGSAC conference on computer
  and communications security}, pages 1054--1067, 2014.

\bibitem[Girgis et~al.(2021)Girgis, Data, Diggavi, Kairouz, and
  Suresh]{girgis2021shuffled}
Antonious Girgis, Deepesh Data, Suhas Diggavi, Peter Kairouz, and
  Ananda~Theertha Suresh.
\newblock Shuffled model of differential privacy in federated learning.
\newblock In \emph{International Conference on Artificial Intelligence and
  Statistics}, pages 2521--2529. PMLR, 2021.

\bibitem[Girgis et~al.(2020)Girgis, Data, Diggavi, Kairouz, and
  Suresh]{girgis2020shuffled}
Antonious~M Girgis, Deepesh Data, Suhas Diggavi, Peter Kairouz, and
  Ananda~Theertha Suresh.
\newblock Shuffled model of federated learning: Privacy, communication and
  accuracy trade-offs.
\newblock \emph{arXiv preprint arXiv:2008.07180}, 2020.

\bibitem[Gorbunov et~al.(2021)Gorbunov, Burlachenko, Li, and
  Richtarik]{gorbunov21marina}
Eduard Gorbunov, Konstantin~P. Burlachenko, Zhize Li, and Peter Richtarik.
\newblock {MARINA}: Faster non-convex distributed learning with compression.
\newblock In \emph{International Conference on Machine Learning}, pages
  3788--3798, 2021.

\bibitem[Gray and Stockham(1993)]{Gray1993dithered}
Robert~M. Gray and T.~G. Stockham.
\newblock Dithered quantizers.
\newblock \emph{IEEE Trans. Information Theory}, 39\penalty0 (3):\penalty0
  805--812, May 1993.

\bibitem[Haddadpour et~al.(2021)Haddadpour, Kamani, Mokhtari, and
  Mahdavi]{haddadpour2021federated}
Farzin Haddadpour, Mohammad~Mahdi Kamani, Aryan Mokhtari, and Mehrdad Mahdavi.
\newblock Federated learning with compression: Unified analysis and sharp
  guarantees.
\newblock In \emph{International Conference on Artificial Intelligence and
  Statistics (AISTATS)}, pages 2350--2358, 2021.

\bibitem[Horv{\'a}th et~al.(2019)Horv{\'a}th, Kovalev, Mishchenko, Stich, and
  Richt{\'a}rik]{horvath2019stochastic}
Samuel Horv{\'a}th, Dmitry Kovalev, Konstantin Mishchenko, Sebastian Stich, and
  Peter Richt{\'a}rik.
\newblock Stochastic distributed learning with gradient quantization and
  variance reduction.
\newblock \emph{arXiv preprint arXiv:1904.05115}, 2019.

\bibitem[Jin et~al.(2020)Jin, Huang, He, Dai, and Wu]{jin2020stochastic}
Richeng Jin, Yufan Huang, Xiaofan He, Huaiyu Dai, and Tianfu Wu.
\newblock Stochastic-sign sgd for federated learning with theoretical
  guarantees.
\newblock \emph{arXiv preprint arXiv:2002.10940}, 2020.

\bibitem[Kairouz et~al.(2016)Kairouz, Bonawitz, and
  Ramage]{kairouz2016discrete}
Peter Kairouz, Keith Bonawitz, and Daniel Ramage.
\newblock Discrete distribution estimation under local privacy.
\newblock In \emph{International Conference on Machine Learning}, pages
  2436--2444. PMLR, 2016.

\bibitem[Kairouz et~al.(2021)Kairouz, Liu, and Steinke]{kairouz2021distributed}
Peter Kairouz, Ziyu Liu, and Thomas Steinke.
\newblock The distributed discrete gaussian mechanism for federated learning
  with secure aggregation.
\newblock In \emph{International Conference on Machine Learning}, pages
  5201--5212. PMLR, 2021.

\bibitem[Kasiviswanathan et~al.(2011)Kasiviswanathan, Lee, Nissim,
  Raskhodnikova, and Smith]{kasiviswanathan2011can}
Shiva~Prasad Kasiviswanathan, Homin~K Lee, Kobbi Nissim, Sofya Raskhodnikova,
  and Adam Smith.
\newblock What can we learn privately?
\newblock \emph{SIAM Journal on Computing}, 40\penalty0 (3):\penalty0 793--826,
  2011.

\bibitem[Kone{\v{c}}n{\`y} et~al.(2016)Kone{\v{c}}n{\`y}, McMahan, Yu,
  Richt{\'a}rik, Suresh, and Bacon]{konevcny2016federated}
Jakub Kone{\v{c}}n{\`y}, H~Brendan McMahan, Felix~X Yu, Peter Richt{\'a}rik,
  Ananda~Theertha Suresh, and Dave Bacon.
\newblock Federated learning: Strategies for improving communication
  efficiency.
\newblock \emph{arXiv preprint arXiv:1610.05492}, 2016.

\bibitem[Krizhevsky et~al.(2009)Krizhevsky, Hinton,
  et~al.]{krizhevsky2009learning}
Alex Krizhevsky, Geoffrey Hinton, et~al.
\newblock Learning multiple layers of features from tiny images.
\newblock 2009.

\bibitem[Lipshitz et~al.(1992)Lipshitz, Wannamaker, and
  Vanderkooy]{Lipshitz1992quantization}
Stanley~P. Lipshitz, Robert~A. Wannamaker, and John Vanderkooy.
\newblock Quantization and dither: A theoretical survey.
\newblock \emph{Journal of the Audio Engineering Society}, 40\penalty0
  (5):\penalty0 355--375, May 1992.

\bibitem[McMahan et~al.(2017)McMahan, Moore, Ramage, Hampson, and
  y~Arcas]{mcmahan2017communication}
Brendan McMahan, Eider Moore, Daniel Ramage, Seth Hampson, and Blaise~Aguera
  y~Arcas.
\newblock Communication-efficient learning of deep networks from decentralized
  data.
\newblock In \emph{Artificial intelligence and statistics}, pages 1273--1282.
  PMLR, 2017.

\bibitem[Mironov(2017)]{mironov2017renyi}
Ilya Mironov.
\newblock R{\'e}nyi differential privacy.
\newblock In \emph{2017 IEEE 30th computer security foundations symposium
  (CSF)}, pages 263--275. IEEE, 2017.

\bibitem[Oyallon and Mallat(2015)]{oyallon2015deep}
Edouard Oyallon and St{\'e}phane Mallat.
\newblock Deep roto-translation scattering for object classification.
\newblock In \emph{Proceedings of the IEEE Conference on Computer Vision and
  Pattern Recognition}, pages 2865--2873, 2015.

\bibitem[Papernot et~al.(2020)Papernot, Thakurta, Song, Chien, and
  Erlingsson]{papernot2020tempered}
Nicolas Papernot, Abhradeep Thakurta, Shuang Song, Steve Chien, and Ulfar
  Erlingsson.
\newblock Tempered sigmoid activations for deep learning with differential
  privacy.
\newblock \emph{arXiv preprint arXiv:2007.14191}, page~10, 2020.

\bibitem[R{\'e}nyi(1961)]{renyi1961measures}
Alfr{\'e}d R{\'e}nyi.
\newblock On measures of entropy and information.
\newblock In \emph{Proceedings of the Fourth Berkeley Symposium on Mathematical
  Statistics and Probability, Volume 1: Contributions to the Theory of
  Statistics}, volume~4, pages 547--562. University of California Press, 1961.

\bibitem[Schuchman(1964)]{Schuchman1964dither}
Leonard Schuchman.
\newblock Dither signals and their effect on quantization noise.
\newblock \emph{IEEE Trans. Communication Technology}, 12\penalty0
  (4):\penalty0 162--165, Dec. 1964.

\bibitem[Seide et~al.(2014)Seide, Fu, Droppo, Li, and Yu]{seide2014one}
Frank Seide, Hao Fu, Jasha Droppo, Gang Li, and Dong Yu.
\newblock 1-bit stochastic gradient descent and its application to
  data-parallel distributed training of speech {DNNs}.
\newblock In \emph{Fifteenth Annual Conference of the International Speech
  Communication Association}, 2014.

\bibitem[Shlezinger et~al.(2020)Shlezinger, Chen, Eldar, Poor, and
  Cui]{Shlezinger2020uveqfed}
Nir Shlezinger, Mingzhe Chen, Yonina~C. Eldar, H.~Vincent Poor, and Shuguang
  Cui.
\newblock {UVeQFed}: Universal vector quantization for federated learning.
\newblock \emph{IEEE Trans. Signal Processing}, 69:\penalty0 500--514, Dec.
  2020.

\bibitem[Sinha and Zoltners(1979)]{sinha1979multiple}
Prabhakant Sinha and Andris~A Zoltners.
\newblock The multiple-choice knapsack problem.
\newblock \emph{Operations Research}, 27\penalty0 (3):\penalty0 503--515, 1979.

\bibitem[Stich and Karimireddy(2020)]{stich2020error}
Sebastian~U. Stich and Sai~Praneeth Karimireddy.
\newblock The error-feedback framework: Better rates for {SGD} with delayed
  gradients and compressed communication.
\newblock \emph{Journal of Machine Learning Research}, 21\penalty0
  (237):\penalty0 1--36, 2020.

\bibitem[Tramer and Boneh(2020)]{tramer2020differentially}
Florian Tramer and Dan Boneh.
\newblock Differentially private learning needs better features (or much more
  data).
\newblock \emph{arXiv preprint arXiv:2011.11660}, 2020.

\bibitem[Warner(1965)]{warner1965randomized}
Stanley~L Warner.
\newblock Randomized response: A survey technique for eliminating evasive
  answer bias.
\newblock \emph{Journal of the American Statistical Association}, 60\penalty0
  (309):\penalty0 63--69, 1965.

\bibitem[Wu and He(2018)]{wu2018group}
Yuxin Wu and Kaiming He.
\newblock Group normalization.
\newblock In \emph{Proceedings of the European conference on computer vision
  (ECCV)}, pages 3--19, 2018.

\end{thebibliography}

\newpage
\appendix
\onecolumn
\section{Proofs}
\label{sec:proofs}

\begin{proof}(Of Lemma~\ref{lem:unbiased})
Observe that in this case:
\begin{align*}
\mathbb{E}(\calA_n(\calM(x_1), \ldots, \calM(x_n)) &= \frac{1}{n} \sum_{i=1}^{n} \mathbb{E} (\calM(x_i)) \\
&= \frac{1}{n} \sum_{i=1}^{n} x_i = \calT(x_1, \ldots, x_n).
\end{align*}
Additionally,
\begin{equation*}
\mathbb{E}\left[ \left(\calA_n(\calM(x_1), \ldots, \calM(x_n)) - \frac{1}{n} \sum_i x_i \right)^2 \right] = \frac{1}{n} \sum_i \mathbb{E}( \calM(x_i) - x_i)^2.\end{equation*}
If $\calM$ has bounded variance, then the variance of $\calA_n(\calM(x_1), \ldots, \calM(x_n))$ diminishes with $n$. The rest of the lemma follows by an application of the Chebyshev's inequality. 
\end{proof}

\begin{proof}(Of Lemma~\ref{lem:metrictovector})
The proof generalizes the argument that the Laplace mechanism applied independently to each coordinate is differentially private for vectors with bounded $L_1$-sensitivity. Let $\bx \in \mathbb{R}^d$ with $\| \bx \|_p \leq \Delta$, and let $Q_0, Q_1$ be density functions for the output distributions of $\calM$ with or without the input $\bx$. Then for any output value $\bz$:
\begin{align*}
    \frac{Q_0(\mathbf{z})}{Q_1(\mathbf{z})} &= \prod_{i=1}^d \frac{Q_0(\mathbf{z}_i)}{Q_1(\mathbf{z}_i)} \\
    &\leq \prod_{i=1}^d \exp(\epsilon \bx_i^p) \quad \text{since $\calM$ is $\epsilon$-metric DP w.r.t. $d(y,y') = |y-y'|^p$}\\
    &= \exp \left( \epsilon \sum_{i=1}^d \bx_i^p \right) \leq \exp(\epsilon \Delta^p).
\end{align*}
The reverse inequality can be derived similarly. Unbiasedness follows from the fact that $\calM$ is unbiased for each dimension $i=1,\ldots,d$.
\end{proof}

\begin{proof}(Of Lemma~\ref{lem:greedy})
Let $\bp$ be a feasible solution for \eqref{eq:comb_opt_lp}. Let $\odot$ and $^{\odot -1}$ denote element-wise product and inverse, respectively. Then:
\begin{align*}
    \sum_{l=1}^d \langle D^\alpha, \bp_l \rangle_F &= \sum_{l=1}^d \langle C \odot (D^\alpha \odot C^{\odot -1}), \bp_l \rangle_F \\
    &\leq \sum_{l=1}^d \langle (D_{i^* j^*}^\alpha / C_{i^* j^*}) C , \bp_l \rangle_F \\
    &\leq (D_{i^* j^*}^\alpha / C_{i^* j^*}) (B-1)^p \Delta^p \\
    &= d_0 D_{i^* j^*}^\alpha.
\end{align*}
\end{proof}

\begin{theorem}\label{thm:umrr}
Unbiased Bitwise Randomized Response satisfies $\epsilon$-local DP and is unbiased.
\end{theorem}
\begin{proof}
By standard proofs of the Randomized Response mechanism, transmitting bit $j$ of $z$ is $\epsilon/b$-differentially private. The entire procedure is thus $\epsilon$-differentially private by composition.

	To show unbiasedness, first we observe that $\bbE[z] = x$. Additionally, let us write $z = \sum_{j=0}^{b-1} 2^{-j} z_j$. The transmitted number $t$ is decoded as $t = \sum_{j=0}^{b-1} 2^{-j} t_j$. Thus, if $\bbE[t_j] = z_j$, then the entire algorithm is unbiased. Observe that:
	\begin{align*}
		\bbE[t_j] &  = a_0 + (a_1 - a_0) \bbE[y_j] \\
		&= a_0 + (a_1 - a_0) \left( z_j \cdot \frac{1}{1 + e^{-\epsilon/b}} + (1 - z_j) \cdot \frac{e^{-\epsilon/b}}{1 + e^{-\epsilon/b}} \right) \\
		&= z_j.
	\end{align*}
	where the last step follows from some algebra. The theorem follows by noting that $\bbE[z] = x$ from properties of dithering.
\end{proof}

\begin{theorem}\label{thm:rappor}
Unbiased Generalized Randomized Response satisfies $\epsilon$-local DP and is unbiased.
\end{theorem}
\begin{proof}
The proof of privacy follows from standard proofs of the privacy of the Generalized RR mechanism. 
To prove unbiasedness, observe that when $z = \frac{i}{B-1}$, the expected output is $a_i$ with probability $\frac{e^{\epsilon}}{B + e^{\epsilon} - 1}$ and $a_j$ for $j \neq i$ with probability $\frac{1}{B + e^{\epsilon} - 1}$. From Equation~\ref{eqn:airappor}, this expectation is also $\frac{i}{B-1} = z$. Additionally, from properties of the dithering process, $\bbE[z] = x$. The unbiasedness follows by combining these two.  
\end{proof}

\section{Experimental Details}
\label{sec:experiment_details}

\subsection{Vector dithering}

The optimization program in the MVU Mechanism operates on numbers on a discrete grid, which are obtained by dithering. In the scalar case, we use the standard dithering procedure on an $x$ in $[0, 1]$. For vectors, we use coordinate-wise dithering on each coordinate. While this leads to an unbiased solution, it might increase the norm of the vector. We show below that the increase in norm is not too high.

\begin{lemma} \label{lem:vector_dithering}
Let $v$ be a vector such that $\|v\| \leq 1$ and $v_i \in [-1, 1]$ for each coordinate $i$. Let $v'$ be the vector obtained by dithering each coordinate of $v$ to a grid of size $B$ (so that the difference between any two grid points is $\Delta = \frac{2}{B-1}$). Then, with probability $\geq 1 - \delta$,
\[ \|v'\|^2 \leq \|v\|^2 +  \sqrt{2} \|v\| \Delta \log(4/\delta) + d \Delta^2/4 + \sqrt{2 d} \Delta \log (4/\delta).   \]
\end{lemma}

\begin{proof}
Let $\Delta = \frac{2}{B-1}$ be the difference between any two grid points. For a coordinate $i$, let $v_i = \lambda_i + a_i$ where $\lambda_i$ is the
closest grid point that is $\leq v_i$ and $a_i \geq 0$. We also let $v'_i = \lambda_i + Z_i$; observe that by the dithering algorithm, $Z_i \in \{0, \Delta\}$, with $\bbE[Z_i] = a_i$. Additionally, $Var(Z_i) \leq \frac{\Delta^2}{4}$.

Additionally, we observe that $\|v'_i\|^2 = \sum_{i} (\lambda_i + Z_i)^2 = \sum_i \lambda_i^2 + 2 \lambda_i Z_i + Z_i^2$.
By algebra, we get that:
\[ \|v'\|^2 - \|v\|^2 = \sum_{i}(Z_i^2 - a_i^2) + \sum_{i} 2 \lambda_i (Z_i - a_i) \]
We next bound these terms one by one.
To bound the second term, we observe that $\bbE[Z_i] = a_i$ and apply Hoeffding's inequality. This gives us:
\[ \Pr(\sum_i \lambda_i Z_i \geq \sum_i \lambda_i a_i + t) \leq 2 e^{-t^2 / 2 \sum_i \lambda_i^2 \Delta^2} \]
Plugging in $t = \sqrt{2 \sum_i \lambda_i^2} \Delta \log(4/\delta)$ makes the right hand side $\leq \delta/2$. To bound the first term, we again use a Hoeffding's inequality.
\[ \Pr(\sum_i Z_i^2 \geq \sum_i \bbE[Z_i^2] + t) \leq 2e^{-t^2/2d \Delta^2} \]
Plugging in $t = \sqrt{2 d} \Delta \log (4/\delta)$ makes the right hand side $\leq \delta/4$. Therefore, with probability $\geq 1 - \delta$,
\[ \|v'\|^2 \leq \|v\|^2 + \sqrt{2 \sum_i \lambda_i^2} \Delta \log(4/\delta) + \sum_i (\bbE[Z_i]^2 - a_i^2)  +  \sqrt{2 d} \Delta \log (4/\delta) \]
Observe that $\bbE[Z_i^2] - a_i^2 = Var(Z_i) \leq \Delta^2/4$; additionally, $\sum_i \lambda_i^2 \leq \|v\|^2$. Therefore, we get:
\[ \|v'\|^2 \leq \|v\|^2 + \sqrt{2} \|v\| \Delta \log(4/\delta) + d \Delta^2/4 + \sqrt{2 d} \Delta \log (4/\delta) \]

The lemma follows.
\end{proof}

In practice, given an \emph{a priori} norm bound $\| v \| \leq R$ for all input vectors $v$, we estimate a scaling factor $\gamma \in [0, 1]$ and apply dithering to the input $\gamma v$ so that $\| \Dither(\gamma v) \| \leq R$ with high probability. This can be done by choosing a confidence level $\delta > 0$ and solving for $\sup \{\gamma \in [0,1] : \| \Dither(\gamma v) \| \leq R \text{ w.p. } \geq 1-\delta\}$ via binary search. Since dithering is randomized, we can perform rejection sampling until the condition $\| \Dither(\gamma v) \| \leq R$ is met. Doing so incurs a small bias that is insignificant in practical applications. We leave the design of more sophisticated vector dithering techniques that simultaneously preserve unbiasedness and norm bound for future work.

\subsection{Connection between DP and compression}

We highlight an interesting effect on the required communication budget as a result of adding differentially private noise. Figure \ref{fig:p_samples_comparison} shows the optimized sampling probability matrix $P$ for the MVU mechanism with a fixed input quantization level $\bin=5$ and various values of $\bout$. As $\bout$ increases, the overall structure in the matrix $P$ remains nearly the same but becomes more refined. Moreover, in the bottom right plot, it is evident that the marginal benefit to MSE becomes lower as $\bout$ increases. This observation suggests that for a given $\epsilon$, having more communication budget is eventually not beneficial to aggregation accuracy since the amount of information in the data becomes obscured by the DP mechanism and hence requires fewer bits to communicate.

\begin{figure}
\centering
\includegraphics[width=0.8\linewidth]{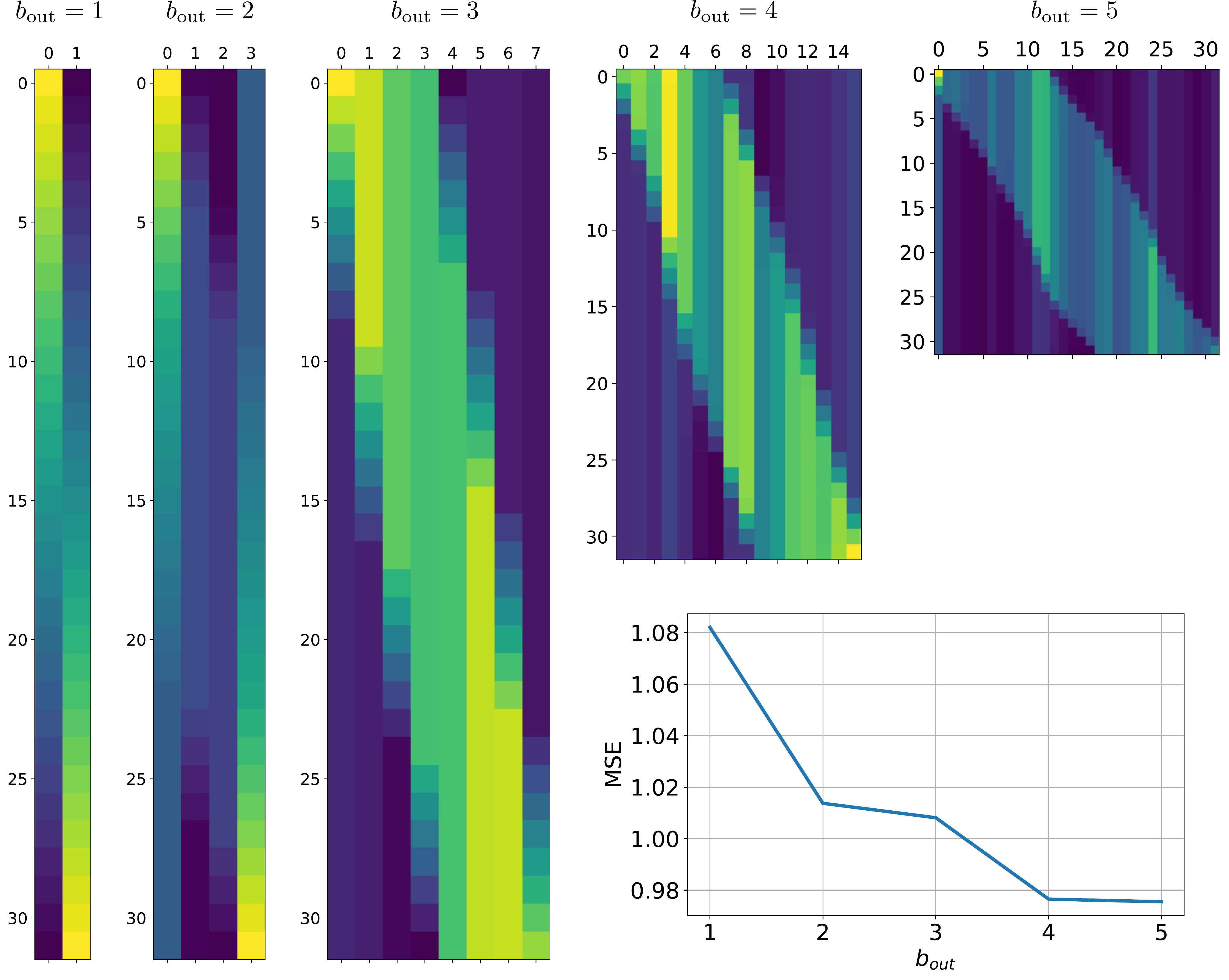}  
\caption{Optimized sampling probability matrix $P$ for the MVU mechanism with $\bin=5$ and different values of $\bout$. The bottom right plot shows that the marginal benefit of the communication budget $\bout$ to MSE becomes lower as $\bout$ increases.}
\label{fig:p_samples_comparison}
\end{figure}

\subsection{Distributed mean estimation}

For the vector distributed mean estimation experiment in Section \ref{sec:dme}, the different private compression mechanisms used different values of the communication budget $b$. We justify the choice of $b$ as follows.

\paragraph{$L_1$-sensitivity setting.} CLDP outputs a \emph{total} number of $\log_2(d) + 1 = 8$ bits, which is lower than that of both Skellam and MVU and cannot be tuned. Skellam performs truncation to the range $\{-2^{b-1}, 2^{b-1}-1\}$ after perturbing the quantized input with Skellam noise, and hence requires a value of $b$ that is large enough to prevent truncation error. We intentionally afforded Skellam a large budget of $b=16$ so that truncation error rarely occurs, and show that even in this setting MVU can outperform Skellam in terms of estimation MSE. For MVU, we chose $\bin=9$, which is the minimum value required to avoid a large quantization error, and $b=\bout=3$.

\paragraph{$L_2$-sensitivity setting.} CLDP uses a communication budget of $b=\log_2(d)+1=8$ \emph{per coordinate} and is not tunable. We used the same $b=16$ budget for Skellam as in the $L_1$-sensitivity setting. For MVU, we chose $\bin=5$ and $b=\bout=3$ for both the $L_1$- and $L_2$-metric DP versions, which results in a communication budget that is lower than both CLDP and Skellam. For the $L_1$-metric DP version, we found that optimizing MVU to satisfy $(\epsilon/2)$-metric DP with respect to the $L_1$ metric results in an $(\epsilon',\delta)$-DP mechanism with $\epsilon' \approx \epsilon$ and $\delta=1/(n+1)$ after optimal RDP conversion.

\subsection{Private SGD}

In Section \ref{sec:fl}, we trained a linear model on top of features extracted by a scattering network\footnote{We used the Kymatio library \url{https://github.com/kymatio/kymatio} to implement the scattering transform.} on the MNIST dataset. In addition, we consider a convolutional network with $\tanh$ activation, which has been found to be more suitable for DP-SGD~\citep{papernot2020tempered}. We give the architecture details of both models in Tables \ref{tab:scatter_linear} and \ref{tab:convnet}.

\begin{table}[h!]
    \begin{minipage}{.5\linewidth}
      \centering
        \resizebox{\textwidth}{!}{
        \begin{tabular}{ll}
            \toprule
            \textbf{Layer} & \textbf{Parameters} \\
            \midrule
            ScatterNet & Scale $J=2$, $L=8$ angles, depth 2 \\
            GroupNorm~\citep{wu2018group} & 6 groups of 24 channels each \\
            Fully connected & 10 units \\
            \bottomrule
        \end{tabular}
        }
        \caption{Architecture for scatter + linear model.}
        \label{tab:scatter_linear}
    \end{minipage}%
    \begin{minipage}{.5\linewidth}
      \centering
        \resizebox{\textwidth}{!}{
        \begin{tabular}{ll}
            \toprule
            \textbf{Layer} & \textbf{Parameters} \\
            \midrule
            Convolution $+\tanh$ & 16 filters of $8 \times 8$, stride 2, padding 2 \\
            Average pooling & $2 \times 2$, stride 1 \\
            Convolution $+\tanh$ & 32 filters of $4 \times 4$, stride 2, padding 0 \\
            Average pooling & $2 \times 2$, stride 1 \\
            Fully connected $+\tanh$ & 32 units \\
            Fully connected $+\tanh$ & 10 units \\
            \bottomrule
        \end{tabular}
        }
        \caption{Architecture for convolutional network model.}
        \label{tab:convnet}
    \end{minipage}
\end{table}

\paragraph{Hyperparameters.} DP-SGD has several hyperparameters, and we exhaustive test all setting combinations to produce the scatter plots in Figures \ref{fig:dpsgd_comparison} and \ref{fig:dpsgd_comparison_convnet}. Tables \ref{tab:hyp_mnist} and \ref{tab:hyp_cifar} give the choice of values that we considered for each hyperparameter.

\begin{table}[h!]
    \begin{minipage}{.47\linewidth}
    \centering
    \resizebox{\textwidth}{!}{
    \begin{tabular}{ll}
        \toprule
        \textbf{Hyperparameter} & \textbf{Values} \\
        \midrule
        Batch size & $600$ \\
        Momentum & $0.5$ \\
        \# Iterations $T$ & $500,1000,2000,3000,5000$ \\
        Noise multiplier $\sigma$ for Gaussian and signSGD & $0.5, 1, 2, 3, 5$ \\
        $L_1$-metric DP parameter $\epsilon$ for MVU & $0.25, 0.5, 0.75, 1, 2, 3, 5$ \\
        Step size $\rho$ & $0.01, 0.03, 0.1$ \\
        Gradient norm clip $C$ & $0.25, 0.5, 1, 2, 4, 8$ \\
        \bottomrule
    \end{tabular}
    }
    \caption{Hyperparameters for DP-SGD on MNIST.}
    \label{tab:hyp_mnist}
    \end{minipage}%
    \begin{minipage}{.53\linewidth}
    \centering
    \resizebox{\textwidth}{!}{
    \begin{tabular}{ll}
        \toprule
        \textbf{Hyperparameter} & \textbf{Values} \\
        \midrule
        Batch size & $500$ \\
        Momentum & $0.5$ \\
        \# Iterations $T$ & $1000,2000,3000,5000,10000,15000$ \\
        Noise multiplier $\sigma$ for Gaussian and signSGD & $0.5, 1, 2, 3, 5$ \\
        $L_1$-metric DP parameter $\epsilon$ for MVU & $0.25, 0.5, 0.75, 1, 2, 3, 5$ \\
        Step size $\rho$ & $0.01, 0.03, 0.1$ \\
        Gradient norm clip $C$ & $0.25, 0.5, 1, 2, 4, 8$ \\
        \bottomrule
    \end{tabular}
    }
    \caption{Hyperparameters for DP-SGD on CIFAR-10.}
    \label{tab:hyp_cifar}
    \end{minipage}
\end{table}

\paragraph{Result for convolutional network.} Figure \ref{fig:dpsgd_comparison_convnet} shows the comparison of DP-SGD training with Gaussian mechanism, stochastic signSGD, and MVU mechanism with $b=1$. The experimental setting is identical to that of Figure \ref{fig:dpsgd_comparison} except for the model being a small convolutional network trained end-to-end. We observe a similar result that MVU recovers the performance of signSGD at equal communication budget of $b=1$.

\begin{figure}
\centering
\includegraphics[width=0.5\linewidth]{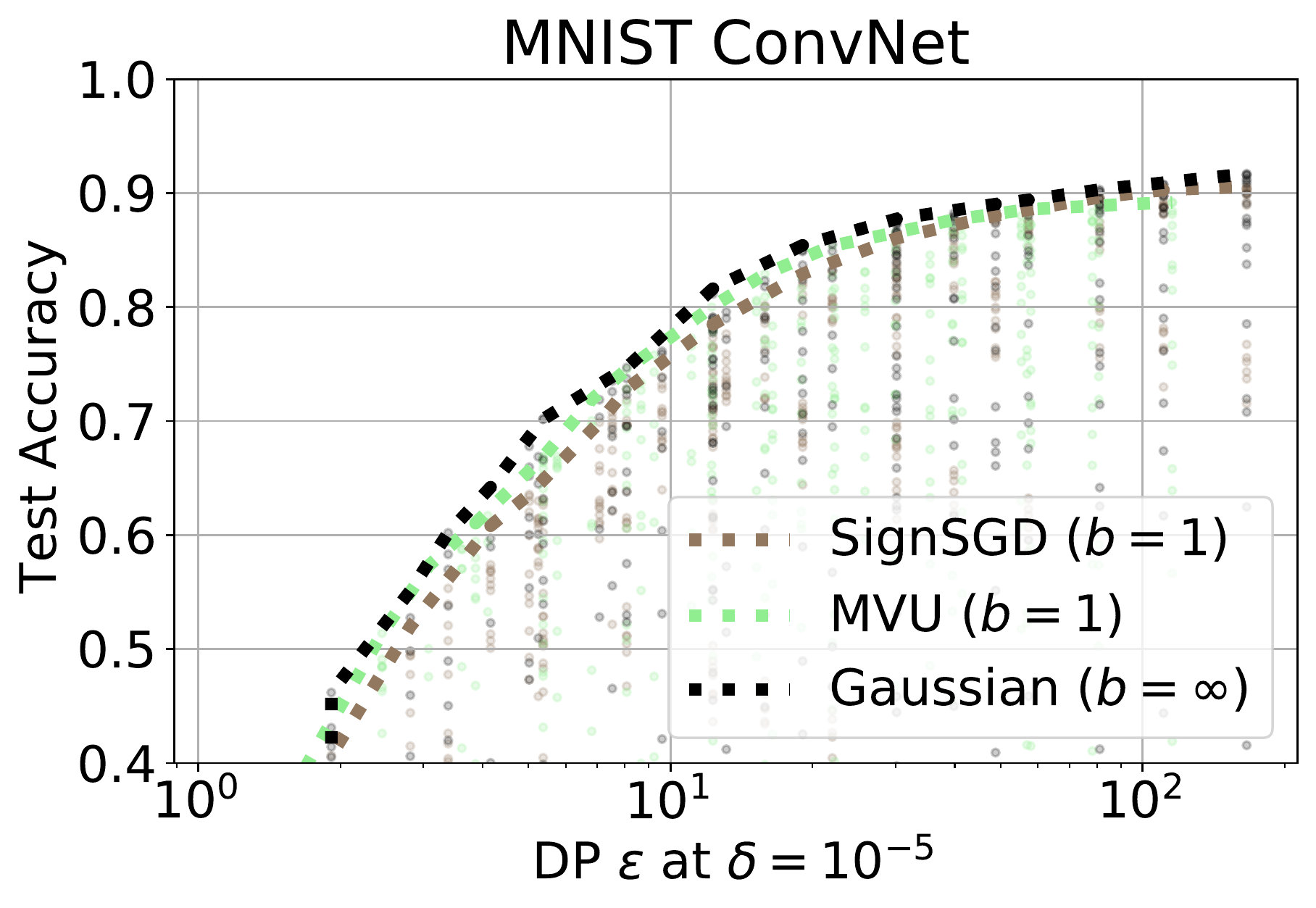}
\caption{DP-SGD training of a small convolutional network on MNIST with Gaussian mechanism, stochastic signSGD and MVU mechanism. Each point corresponds to a single hyperparameter setting, and dashed line shows Pareto frontier of privacy-utility trade-off.}
\label{fig:dpsgd_comparison_convnet}
\end{figure}


\end{document}